\pdfoutput=1
\documentclass[11pt]{article}
\usepackage{smile}
\usepackage{subcaption}
\usepackage{pdflscape}   
\usepackage{pifont}
\usepackage{xcolor}
\usepackage{tikz}
\usetikzlibrary{arrows.meta, positioning, calc, shapes.geometric}
\definecolor{morandiGreen}{RGB}{140,160,140}
\definecolor{morandiPink}{RGB}{244,186,180}  
\definecolor{morandiBlue}{RGB}{190,210,250}  
\definecolor{morandiOrange}{RGB}{200,160,120}
\definecolor{morandiPurple}{RGB}{150,140,160}
\definecolor{morandiGray}{RGB}{160,160,160}
\usepackage[margin=1in]{geometry}
\usepackage[most]{tcolorbox}

\newcommand{\wh}{\widehat}

\title{\bf Diffusion-Based Data-Driven Assortment Optimization}
\author{Junyi Liao$^\dagger$,\ \ 
Xiaohui Jiang$^\ddagger$,\ \ 
Zhengwei Tong$^\S$,\ \ 
Ethan X. Fang$^\ddagger$,\ \ 
Vahid Tarokh$^\dagger$
\\ 
~\\
$^{\dagger}$\small\it Department of Electrical and Computer Engineering, Duke University \\ 
$^{\ddagger}$\small\it Department of Biostatistics and Bioinformatics, Duke University \\ 
$^{\S}$\small\it Department of Computer Science, Duke University \\ \small\sf \{junyi.liao, x.jiang, zhengwei.tong, ethan.fang, vahid.tarokh\}@duke.edu
}
\date{}
\begin{document}
\maketitle

\begin{abstract}
Assortment optimization is a fundamental problem in revenue management, typically addressed using parametric choice models such as the multinomial logit (MNL) and its variants.  While these models enable tractable formulations, their performance is sensitive to model misspecification and often struggles to capture complex customer behavior. In this paper, we propose a model-agnostic framework for assortment optimization based on guided discrete diffusion. We represent assortments as binary vectors and perform stochastic search via a learned reverse diffusion process, avoiding explicit combinatorial enumeration. To incorporate decision objectives, we introduce a reward-guided mechanism that biases local transitions using estimates of expected revenue. This allows the method to effectively balance exploration and exploitation during generation.  Empirically, we show that the proposed approach consistently identifies high-quality assortments and remains robust under model misspecification, often recovering near-optimal solutions in high-dimensional settings. Moreover, the generative nature of diffusion enables the production of diverse high-performing assortments, offering flexibility beyond a single deterministic solution. These results highlight the potential of generative modeling as a scalable and robust paradigm for combinatorial optimization in data-driven decision-making.
\end{abstract}

%%%%%%%%%%%%-- main content --%%%%%%%%%%
\section{Introduction}
Assortment Optimization (AO) is a core problem in operations research, revenue management, and online platforms \citep{mcfadden1972conditional, rooderkerk2013optimizing, talluri2004revenue, rusmevichientong2010dynamic, strauss2018review}. In this setting, the decision-maker selects an assortment—a subset of products to offer to customers—to maximize expected revenue or utility based on customer choice behavior. This problem arises in a wide range of applications, including retail product display, online recommendation systems, advertisement placement, and service bundling \citep{fisher2014demand, liu2025simultaneous}. As digital platforms scale and product catalogs grow, the ability to efficiently optimize assortments has become increasingly critical for both economic performance and user experience.

Despite its practical importance, general assortment optimization, especially under flexible or data-driven choice models, is fundamentally challenging due to both computational and statistical complexities \citep{jagabathula2014assortment, strauss2018review}. At a high level, the difficulty stems from the combinatorial nature of the decision space: the assortment is a subset of items, and the number of feasible choices grows exponentially with the number of products. This challenge is further compounded by the structure of the objective function, which is highly non-linear due to substitution effects—adding or removing a single product can significantly alter the choice probabilities of all others \citep{train2009discrete, farias2013nonparametric, chen2022decision}. As a result, the optimization landscape is complex and often non-convex, rendering exhaustive search or naive optimization infeasible beyond small-scale instances \citep{rusmevichientong2010dynamic, desir2022capacitated}.

Existing approaches typically rely on parametric choice models, such as the multinomial logit (MNL) and its variants \citep{mcfadden1972conditional, manski1981structural, mcfadden2000mixed, kok2008assortment, feng2022consumer}, which enable tractable formulations and specialized algorithms. While effective under correct model specification, these methods suffer from two key limitations. First, they are sensitive to model misspecification, which can lead to suboptimal decisions in the presence of heterogeneous preferences or complex substitution patterns \citep{train2009discrete, farias2013nonparametric, bertsimas2015data, sturt2025value}. Second, even under parametric assumptions, scalable optimization remains challenging for richer models and in large-scale or constrained settings \citep{bront2009column, jagabathula2020conditional, desir2020constrained}. These difficulties are further exacerbated in the offline setting, where only historical interaction data is available. In this case, the decision-maker must both estimate the choice model and optimize the assortment under uncertainty, compounding statistical and computational challenges \citep{dong2025pasta, han2025learning}. Consequently, developing general-purpose methods that are both computationally efficient and robust to model misspecification in the offline setting remains an open challenge.

In parallel, diffusion-based generative models have recently emerged as a powerful paradigm for modeling complex, high-dimensional distributions \citep{ho2020denoising, song2020denoising, song2020score, chen2024overview}. Beyond their success in image and text generation, diffusion models have shown promise as stochastic optimizers, capable of exploring structured and combinatorial solution spaces \citep{sun2023difusco, zhao2024disco}. By iteratively refining noisy candidates, these models can efficiently navigate complex energy landscapes and produce diverse high-quality solutions. This generative perspective is particularly appealing for AO: instead of searching for a single optimum via expensive optimization routines, one can learn to sample near-optimal assortments directly. Such an approach naturally balances exploration and exploitation, and can return multiple high-quality solutions within a short computational budget.

In this paper, we propose \emph{Diffusion-based Data-Driven Assortment Optimization (D3AO)}, which is a generative framework for assortment optimization in the offline setting. Our approach treats assortments as discrete structured objects and learns to model high-reward regions of the solution space via a diffusion process. By incorporating reward-guided refinement into generation, D3AO efficiently produces diverse near-optimal assortments without relying on restrictive parametric assumptions, providing a flexible and scalable alternative to traditional optimization pipelines.

Our main contributions are summarized as follows:
\begin{itemize}[itemsep=-2pt, topsep=2pt]
\item\textit{A new generative formulation of AO.} We model assortment optimization as a sampling problem over combinatorial structures and introduce a diffusion-based framework tailored to discrete decision spaces.
\item\textit{Offline optimization with implicit behavior modeling.} Our method avoids explicit parametric choice model estimation, improving robustness under model misspecification.
\item\textit{Efficient generation of high-quality assortments.} The proposed approach produces multiple near-optimal solutions via stochastic generation, avoiding exhaustive combinatorial search.
\item\textit{Empirical validation in controlled settings.} Through experiments on synthetic benchmarks, we show that our method consistently identifies high-quality assortments and exhibits strong robustness under model misspecification.
\end{itemize}

\subsection{Related Works}
% \paragraph{Assortment Optimization.}
\paragraph{Assortment Optimization.}
Assortment optimization has been widely studied under different models of customer choice \citep{mcfadden1972conditional, talluri2004revenue, kok2008assortment, strauss2018review}. A major line of work focuses on parametric models, such as the multinomial logit and its variants, which enable tractable formulations and efficient algorithms \citep{manski1981structural, rusmevichientong2010dynamic, feng2022consumer}. However, these approaches can be sensitive to model misspecification, particularly in the presence of complex substitution patterns \citep{train2009discrete}. To improve flexibility, subsequent work considers nonparametric and data-driven models that relax structural assumptions on customer behavior \citep{farias2013nonparametric, bertsimas2015data}, though often at the cost of increased computational difficulty. A large body of literature further studies online assortment optimization, where the decision-maker sequentially learns customer preferences while optimizing assortments, typically using bandit-based approaches \citep{rusmevichientong2010dynamic, saure2013optimal, gong2022online, li2025online}. More recently, offline assortment optimization has been studied, where decisions must be made from historical data without active exploration \citep{wang2023neural, dong2025pasta, han2025learning}, further coupling statistical uncertainty with combinatorial optimization. In contrast to these approaches, we adopt a generative perspective and directly learn to produce high-quality assortments from data.

\paragraph{Diffusion-based Generative Models.}
Diffusion models \citep{ho2020denoising, song2020denoising, song2020score} provide a powerful framework for modeling complex distributions via a progressive denoising process. While initially developed for continuous domains, subsequent work extended diffusion to discrete state spaces \citep{austin2021structured, sun2022score}, enabling applications to categorical and combinatorial objects. Recent advances further introduce guidance mechanisms \citep{bansal2023universal, schiff2025simple} to steer generation toward desired outcomes. Building on these developments, diffusion-based methods have been applied to combinatorial optimization \citep{sun2023difusco, sanokowski2024diffusion, zhao2024disco, wang2025fractional, aloui2026score}, where they explore large discrete solution spaces through stochastic generation. Despite promising empirical results, their application to offline assortment optimization remains largely unexplored.

%Diffusion-based generative models \citep{ho2020denoising, song2020denoising, song2020score} model complex distributions through a progressive denoising procedure. While early methods were developed for continuous domains, later works extended diffusion to discrete state spaces \citep{austin2021structured, sun2022score}, making it applicable to categorical and combinatorial objects. Recent studies further introduced guidance mechanisms \citep{bansal2023universal, schiff2025simple} to steer generation toward desirable outputs. Motivated by these developments, diffusion-based approaches have also been proposed for combinatorial optimization \citep{sun2023difusco,sanokowski2024diffusion,zhao2024disco,wang2025fractional,aloui2026score}, where they are used to search over large and discrete solution spaces. While these methods have shown strong performance on generic combinatorial problems, their use in offline assortment optimization has been limited.

\paragraph{Offline Learning.}
Offline learning studies how to make decisions from historical interaction data without active exploration \citep{swaminathan2015counterfactual, prudencio2023survey}. A central challenge in this setting is that only partial feedback is observed under a logging policy, leading to selection bias, distribution shift, and limited coverage of the action space \citep{dudik2011doubly, levine2020offline}. Prior work addresses these issues through off-policy evaluation and counterfactual learning, aiming to evaluate or optimize decision rules from logged data \citep{dudik2011doubly, swaminathan2015counterfactual, joachims2018deep}. More recent advances in offline reinforcement learning further emphasize conservative or pessimistic approaches to mitigate extrapolation error \citep{fujimoto2019off, kumar2020conservative}. While these methods have primarily been studied in contextual bandits and reinforcement learning, our setting additionally involves combinatorial discrete actions and customer choice–dependent rewards, making offline assortment optimization more challenging \citep{shimizu2024effective}.

% Offline learning studies how to make decisions from historical interaction data without active exploration \citep{swaminathan2015counterfactual, prudencio2023survey}. A challenge in this setting is that the learner only observes partial feedback under the logging policy, which can lead to selection bias, distribution shift, and limited coverage of the action space \citep{dudik2011doubly,levine2020offline}. Many existing works address these challenges through off-policy evaluation and counterfactual learning, where the goal is to evaluate or optimize decision rules using logged feedback \citep{dudik2011doubly,swaminathan2015counterfactual}. While these methods have been studied mainly in contextual bandits and reinforcement learning, our setting additionally involves combinatorial discrete actions and customer choice--dependent rewards, making offline assortment optimization more challenging \citep{shimizu2024effective}.

\section{Offline Assortment Optimization}
\paragraph{Problem formulation.}
Assortment optimization is a fundamental problem in revenue management and online retail \citep{rooderkerk2013optimizing,kok2008assortment,qi2020data}. A seller is given a set of products and must decide which subset to display to a customer in order to maximize expected revenue. The challenge arises from the interaction between the offered assortment and the customer's choice behavior: offering more products increases variety but also induces substitution effects that may reduce overall revenue.

Formally, let $[N] = [N]$ denote the set of available products. An \emph{assortment} is defined as a nonempty subset of products, i.e., $s \subseteq [N]$ with $s \neq \emptyset$. Let $\mathscr{S} \subseteq 2^{[N]} \setminus \{\emptyset\}$ denote the feasible set of assortments (e.g., subject to cardinality or business constraints), where the unconstrained case corresponds to $\mathscr{S} = 2^{[N]} \setminus \{\emptyset\}$.

An assortment optimization problem is specified by a quadruple $([N], \mathscr{S}, p, r)$, where:
\begin{itemize}[itemsep=-2pt, topsep=2pt]
\item $[N]$ is the set of available products.
\item $\mathscr{S}\subseteq 2^{[N]}\setminus\{\emptyset\}$ is the feasible set of assortments;
\item $p:\mathscr{S} \to \Delta([N]\cup\{0\})$ is a \emph{choice policy}, where $p(\cdot\,|\, s)$ is a probability distribution supported on $s \cup \{0\}$. This distribution describes the choice behavior of a customer population when exposed to assortment $s$. Here we augment $[N]$ with a special element $0$ representing the \emph{no-purchase} option;
\item $r:\mathscr{S} \times ([N]\cup\{0\}) \to \mathbb{R}$ is the reward function, where $r(s,a)$ specifies the revenue obtained by the seller if item $a$ is chosen under assortment $s$. We assume $r(s,a)=0$ for $a \notin s$ and typically $r(s,0)=0$.
\end{itemize}

In a single episode of assortment deployment, the interaction proceeds as follows: the seller offers an assortment $s \in \mathscr{S}$ to a customer, and the customer selects an item
\begin{equation*}
A \sim p(\cdot \mid s), \qquad A \in s \cup \{0\},
\end{equation*}
where $A=0$ corresponds to the no-purchase outcome. The seller then receives reward $r(s,A)$. The expected revenue of an assortment $s$ is therefore given by
\begin{equation*}
R(s)
=
\mathbb{E}_{A \sim p(\cdot \mid s)}[r(s,A)]
=
\sum_{a \in [N] \cup \{0\}} r(s,a)\, p(a \mid s)
=
\sum_{a \in s} r(s,a)\, p(a \mid s).
\end{equation*}
The goal of assortment optimization is to identify an optimal feasible assortment $s^\star$ that maximizes the expected revenue:
\begin{equation*}
s^\star
\in
\argmax_{s \in \mathscr{S}} R(s).
\end{equation*}
This problem is challenging for two main reasons. First, the decision space is combinatorial, with $|\mathscr{S}|$ growing exponentially in $N$. Second, the reward function is defined through the interaction between $p$ and $r$, inducing complex, non-linear substitution effects across products. These challenges make exact optimization intractable beyond small-scale settings. In general, when the choice behavior is flexible or data-driven, the resulting optimization problem becomes computationally intractable (often NP-hard), and admits no efficient exact solution beyond small-scale instances \citep{bront2009column,jagabathula2014assortment}. These challenges motivate the need for scalable approximate optimization methods.

\paragraph{Offline learning setting.}
We consider an offline setting where only historical interaction data is available \citep{wang2023neural, dong2025pasta, han2025learning}. Specifically, we are given a dataset
\begin{equation*}
\mathcal{D} = \{(S_i, A_i)\}_{i=1}^n,
\end{equation*}
where each $S_i \in \mathscr{S}$ is a previously offered assortment and $A_i \in S_i \cup \{0\}$ is the observed customer choice. We assume that the reward function $r(s,a)$ is known, which is standard in the literature (e.g., prices or margins are typically available).

A key challenge in the offline setting is that the data is generated by a behavior policy $\mu$ over assortments, rather than being uniformly sampled from $\mathscr{S}$. As a result, the dataset typically provides only partial coverage of the combinatorial space, leading to a distribution shift between the observed assortments and those that may be optimal. This makes reliable evaluation and optimization of unseen assortments inherently difficult.

\subsection{A Maximum Entropy Perspective on Assortment Optimization}
In the offline setting described above, the observed assortments $\{S_i\}_{i=1}^n$ are generated by a behavioral policy $\mu$ over the feasible set $\mathscr{S}$. In practice, these assortments are not collected uniformly at random. Instead, they are typically produced by existing decision-making processes, such as heuristics, learned policies, or human operators. These decision-makers are often \emph{imperfect but reasonable}: they aim to favor high-reward assortments, but do not consistently identify the global optimum and may retain a degree of stochasticity or exploration \citep{ziebart2008maximum, levine2020offline}. As a result, the observed data tends to be biased toward higher-reward regions of the combinatorial space, while still maintaining some diversity.

To capture this behavior, we assume that the behavior policy follows a Boltzmann (softmax) distribution:
\begin{equation}
\mu(s)
\propto
\exp\bigl(\beta_0 R(s)\bigr),\quad s\in\mathscr{S}\label{eq:boltzmann}
\end{equation}
for some inverse temperature $\beta_0 > 0$. 

The parameter $\beta_0$ in \eqref{eq:boltzmann} controls the level of rationality of the behavior policy. When $\beta_0$ is small, the distribution becomes nearly uniform over $\mathscr{S}$, corresponding to a highly exploratory or random policy. As $\beta_0$ increases, the distribution becomes increasingly concentrated on high-reward assortments, approaching a deterministic optimal policy in the limit $\beta_0 \to \infty$. In particular, moderate values of $\beta_0$ yield a balanced regime in which the data exhibits both quality and diversity, avoiding degenerate scenarios where the data is generated either by a near-optimal oracle (yielding limited coverage) or by a random policy (yielding weak learning signals).

Importantly, the Boltzmann form in \eqref{eq:boltzmann} is not merely a modeling assumption. It admits a principled interpretation as the solution to an \textit{entropy-regularized revenue maximization problem}. In other words, such a policy arises as a \emph{soft-optimal} decision rule that balances revenue
maximization with entropy, providing a theoretical justification for its use in modeling historical data. We formalize this insight in the following theorem.

\begin{theorem}[Entropy-regularized characterization of Boltzmann policies]\label{thm:entreg}
Let $\mathscr{S}\subset\{0,1\}^N$ be a finite set of feasible assortments and $R : \mathscr{S} \to \mathbb{R}$
be a reward function. Consider the entropy-regularized revenue maximization problem
\begin{equation*}
\max_{q \in \Delta(\mathscr{S})}
\; \mathbb{E}_{s \sim q}[R(s)] + \frac{1}{\beta} \mathcal{H}(q),
\end{equation*}
where $\mathcal{H}(q) = -\sum_{s \in \mathscr{S}} q(s)\log q(s)$ denotes the \textit{Shannon entropy} \citep{shannon1948mathematical} of $q$, and $\beta > 0$ is a regularization parameter. Then the unique maximizer is given by the Boltzmann distribution
\begin{equation*}
q^\star(s) = \frac{\exp(\beta R(s))}{\sum_{s' \in \mathscr{S}} \exp(\beta R(s'))},\quad s\in\mathscr{S}.
\end{equation*}
\end{theorem}
\begin{proof}
See Appendix \S\ref{sec:entreg} for a detailed proof.
\end{proof}

Theorem~\ref{thm:entreg} shows that Boltzmann policies arise as solutions to entropy-regularized optimization, where the entropy term promotes diversity over the feasible set of assortments. This perspective will be central to our approach, as it suggests modeling assortment optimization as the problem of constructing a distribution that balances revenue maximization and exploration.

In the offline setting, our goal is not only to learn the underlying choice behavior from data, but more importantly to improve upon the historical policy by identifying higher-reward assortments. This naturally leads to a stochastic optimization viewpoint: rather than directly solving a combinatorial
maximization problem, we aim to construct a distribution over assortments that concentrates on high-reward regions while maintaining sufficient diversity. Diffusion-based generative models provide a principled and scalable mechanism for this purpose. By learning a data-driven prior over assortments and incorporating reward-based guidance, they enable efficient exploration of the combinatorial space while progressively biasing samples toward high-reward solutions.

\paragraph{Connection to maximum-entropy IRL.}
The Boltzmann behavior policy assumption is closely related to
prior work on maximum-entropy inverse reinforcement learning
\citep{ziebart2008maximum, wulfmeier2015maximum, snoswell2020revisiting}, where trajectories are distributed according to
\begin{equation*}
p(\tau) \propto q(\tau)\exp\bigl(R(\tau)\bigr),
\end{equation*}
with $q(\tau)$ denoting the marginal probability induced by the transition dynamics, and $R(\tau)$ the cumulative reward along $\tau$. In our setting, this corresponds to a degenerate one-step decision problem in which assortments play the role of actions and the base measure is implicit. Under this view, the historical data can be interpreted as generated by a soft-optimal policy over assortments. Our objective is therefore to construct
an improved distribution that shifts probability mass toward higher-reward assortments while retaining sufficient exploration.

\section{D3AO: Diffusion-based Data-Driven Assortment Optimization}
In this section, we propose \emph{Diffusion-based Data-Driven Assortment Optimization} (D3AO), a three-stage framework that integrates choice model estimation with guided generative optimization for offline assortment selection. 
An overview of the framework is illustrated in Figure~\ref{fig:d3ao-pipeline}. 
At a high level, D3AO consists of three components: (i) learning a neural choice model from offline data, (ii) constructing a model-based reward estimator, and (iii) performing guided stochastic optimization via a diffusion-based generative model. 
The diffusion model serves as a data-driven prior over feasible assortments, while the reward signal acts as a guiding force that progressively steers samples toward high-reward regions.

\begin{tcolorbox}[
    colback=gray!10,
    colframe=black,
    title=Step 1: Choice model estimation,
    coltitle=white,
    colbacktitle=black,
    fonttitle=\bfseries,
    boxrule=1pt, arc=2pt
]
Given offline data $\mathcal{D}=\{(S_i, A_i)\}_{i=1}^n$, we learn a neural choice model $p_\theta(a\,|\, s)$ that maps an assortment $s \in \mathscr{S}$ to a probability distribution over $[N] \cup \{0\}$. The model is constrained such that its output is supported on $s \cup \{0\}$, i.e., $p_\theta(a \mid s)=0$ for $a \notin s \cup \{0\}$.
We train $p_\theta$ by minimizing the \textit{cross-entropy loss} on observed choices:
\begin{equation*}
\hat{\theta}
\in
\argmin_{\theta}
-\frac{1}{n} \sum_{i=1}^n \log p_\theta(A_i \mid S_i),
\end{equation*}
which corresponds to Maximum Likelihood Estimation (MLE) over a flexible function class. In practice, we adopt neural choice models in \cite{wang2023neural}, where the model takes the assortment as input and outputs a masked probability distribution over feasible items.
\end{tcolorbox}

This formulation allows us to capture complex and potentially confounded customer behavior, including non-linear substitution patterns and interactions that are difficult to model with traditional parametric approaches such as multinomial logits.

\begin{figure}[t]
    \centering
    \includegraphics[width=\textwidth]{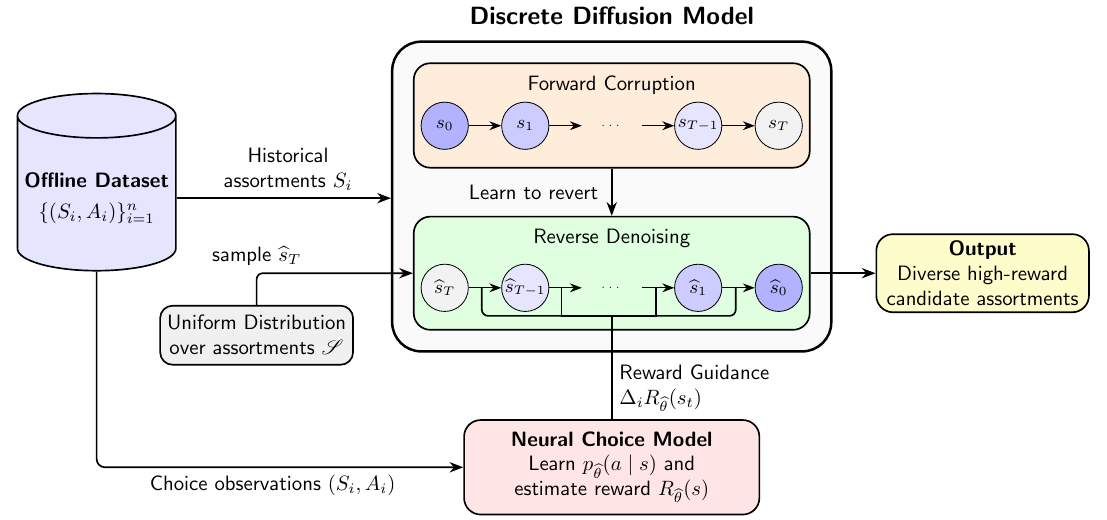}
    \caption{Our D3AO framework for offline assortment optimization. Starting from an offline dataset of historical assortments and customer choices, we first learn a neural choice model to estimate the expected reward of candidate assortments. In parallel, a discrete diffusion model is trained on historical assortments through a forward corruption process and a reverse denoising process. At inference time, the reverse process is initialized from a uniform distribution over assortments and is guided by the learned reward signal, producing diverse high-reward candidate assortments.}
    \label{fig:d3ao-pipeline}
\end{figure}

\begin{tcolorbox}[
    colback=gray!10,
    colframe=black,
    title=Step 2: Reward estimation,
    coltitle=white,
    colbacktitle=black,
    fonttitle=\bfseries,
    boxrule=1pt, arc=2pt
]
Using the fitted choice model, we construct a model-based (plug-in) estimator of the expected reward for any assortment $s$:
\begin{equation*}
R_{\wh\theta}(s)
=
\sum_{a \in s\,\cup\,\{0\}}
r(s,a)\, p_{\hat{\theta}}(a \mid s).
\end{equation*}
Essentially, this estimator approximates the true expected reward $R(s)$ by replacing the unknown choice policy $p$ with the learned model $p_{\hat{\theta}}$.
\end{tcolorbox}

The plug-in form enables efficient evaluation of arbitrary assortments without requiring additional data collection, which is crucial in the offline setting. Moreover, it allows us to generalize beyond the support of the observed data by leveraging the learned structure of the choice model. This provides a tractable surrogate objective for optimization over the combinatorial space.

Since the assortments $(S_i)$ in offline data are generated by a Boltzmann-like behavior policy $\mu(s) \propto \exp(\beta_0 R(s))$, which favors high-reward assortments while retaining stochasticity, we adopt a generative approach to approximate and improve this policy.
\begin{tcolorbox}[
    colback=gray!10,
    colframe=black,
    title=Step 3: Guided generative optimization from the behavior prior,
    coltitle=white,
    colbacktitle=black,
    fonttitle=\bfseries,
    boxrule=1pt, arc=2pt
]

\begin{itemize}
\item[\bf (I)] \textbf{Learning a generative behavior prior.}
We first train a discrete diffusion model \citep{austin2021structured} on the offline assortments $\{S_i\}_{i=1}^n$ to learn a generative approximation of the behavior policy. The forward corruption process gradually perturbs assortments toward a simple reference distribution:
\begin{equation*}
\mu = \mu_0 
\;\rightarrow\; \mu_1 
\;\rightarrow\; \cdots 
\;\rightarrow\; \mu_T \approx \text{Uniform}(\{0,1\}^N),
\end{equation*}
while the learned reverse process generates samples by approximately inverting this trajectory:
\begin{equation*}
\text{Uniform}(\{0,1\}^N) = \hat{\mu}_T 
\;\rightarrow\; \hat{\mu}_{T-1} 
\;\rightarrow\; \cdots 
\;\rightarrow\; \hat{\mu}_0 = \hat{\mu}.
\end{equation*}
This reverse process provides a generative approximation $\hat{\mu}$ of the behavior policy, capturing structural regularities of historically offered assortments.

\item[\bf (II)]\textbf{Reward-guided policy improvement.}
Sampling from $\hat{\mu}$ alone mainly reproduces the historical policy. To improve upon it, we guide the reverse denoising process using the learned choice model $p_{\hat\theta}$ in Step I, equivalently the estimated reward $R_{\hat{\theta}}(s)$. At each reverse step $\wh{\mu}_t\to\wh{\mu}_{t-1}$, the transition is modulated by a reward-dependent signal that biases the generation toward higher-value assortments.
\end{itemize}
Further details of the diffusion formulation and guidance mechanism are provided in \S\ref{sec:diffusion}.
\end{tcolorbox}

Intuitively, the unguided diffusion model defines a data-driven prior over plausible assortments, while the reward-based guidance acts as an external force that drives samples toward high-reward (low-energy) regions. The resulting procedure balances exploration and exploitation: the diffusion prior preserves diversity and prevents purely greedy search, whereas the guidance progressively concentrates samples in promising regions of the combinatorial space.

We emphasize that our goal is not to exactly sample from a prescribed Boltzmann distribution, but to perform guided policy improvement by combining a learned generative approximation of the behavior policy with reward-driven refinement. The method naturally produces a diverse set of high-quality assortments, which can be directly used in downstream applications. Additional selection can be applied if desired, but is not required.

\section{Guided Discrete Diffusion}\label{sec:diffusion}
In this section, we introduce a discrete diffusion framework for assortment optimization and describe how it can be guided by learned reward signals for policy improvement. Our approach adopts a generative modeling perspective: rather than solving the combinatorial optimization problem directly, we learn a structured distribution over assortments from offline data and then bias the sampling process toward high-reward regions.

The main challenge in assortment optimization lies in the exponentially large decision space. To address this, we model assortments as binary vectors and use a discrete diffusion process to generate candidate solutions through iterative denoising. This provides a flexible mechanism for exploring high-dimensional combinatorial spaces while retaining the structural information present in historical data. Throughout this section, we focus on the unconstrained setting $\mathscr{S}=\{0,1\}^N$, where each assortment is represented by a binary vector $s$. This formulation allows us to model the combinatorial space in a unified and tractable way.

While discrete diffusion models have been studied in prior work \citep{austin2021structured}, our setting differs in that the goal is not exact generative modeling, but reward-guided combinatorial optimization.  In particular, we develop a guided diffusion procedure that integrates learned choice models into the reverse process, allowing us to systematically bias generation toward high-reward regions while preserving diversity.

\subsection{Denoising Discrete Diffusion}
We first introduce the underlying denoising diffusion model over binary assortments, which serves as the generative backbone of our approach.
\paragraph{Forward process (corruption).}
We adopt a discrete-time Bernoulli corruption process that progressively randomizes an assortment vector. Starting from a clean assortment $s_0 \in \{0,1\}^N$, the forward process generates a sequence $\{s_t\}_{t=0}^T$ according to
\begin{equation}
q(s_{1:T}\mid s_0)
=
\prod_{t=1}^T q_t(s_t\mid s_{t-1}),
\label{eq:forward}
\end{equation}
where each transition acts independently across coordinates. For each $i\in[N]$,
\begin{equation*}
q_t(s_{t,i}\mid s_{t-1,i})
=
(1-\beta_t)\mathbf{1}\{s_{t,i}=s_{t-1,i}\}
+\frac{\beta_t}{2},
\qquad s_{t,i}\in\{0,1\},
\end{equation*}
with $\beta_t\in(0,1)$ a prescribed noise schedule. Equivalently, with probability $\beta_t$, the $i$-th bit is resampled uniformly from $\{0,1\}$; otherwise it remains unchanged. Hence the transition factorizes as
\begin{equation*}
q_t(s_t\mid s_{t-1})
=
\prod_{i=1}^N q_t(s_{t,i}\mid s_{t-1,i}),
\qquad s_t\in\mathscr{S}.
\end{equation*}
As $t$ increases, the forward process gradually destroys the combinatorial structure of the original assortment and approaches a near-uniform distribution over $\{0,1\}^N$.

\paragraph{Denoising parameterization.}
To reverse the corruption process, we learn a parametric reverse model that reconstructs structured assortments from noisy inputs. In principle, this defines a reverse generative process
\begin{equation*}
p_\phi(s_{0:T})
=
p(s_T)\prod_{t=1}^T p_\phi(s_{t-1}\mid s_t,t),
\end{equation*}
where $p(s_T)$ is chosen as the uniform distribution over $\{0,1\}^N$. Rather than parameterizing each reverse transition $p_\phi(s_{t-1}\mid s_t,t)$ directly, we follow Austin et al.~\citep{austin2021structured} and let the neural network predict the clean assortment $s_0$ from the noisy state $(s_t,t)$. 

Specifically, given a noisy assortment $s_t$ and timestep $t$, the network outputs
\begin{equation*}
g_\phi(s_t,t)\in\mathbb{R}^N,
\end{equation*}
which we interpret as logits of the clean assortment conditioned on $(s_t,t)$. These logits define Bernoulli probabilities
\begin{equation*}
\alpha_{t,i}=\sigma\!\bigl(g_\phi^{(i)}(s_t,t)\bigr),
\qquad i=1,\dots,N,
\end{equation*}
and hence a factorized distribution over the clean sample:
\begin{equation*}
\tilde p_\phi(\tilde s_0\mid s_t,t)
=
\prod_{i=1}^N
\alpha_{t,i}^{\tilde s_{0,i}}(1-\alpha_{t,i})^{1-\tilde s_{0,i}}.
\end{equation*}

The model is trained via denoising. Let $\mathcal{D}$ denote the empirical distribution of assortments in the offline dataset. We sample
\begin{equation*}
s_0\sim\mathcal{D},
\qquad
t\sim\mathrm{Unif}\{1,\cdots,T\},
\qquad
s_t\sim q_t(\cdot\mid s_0),
\end{equation*}
and train the network to predict the clean assortment $s_0$ from $(s_t,t)$ by minimizing the binary cross-entropy loss
\begin{equation*}
\mathcal{L}(\phi)
=
\mathbb{E}_{s_0\sim\mathcal{D}}
\mathbb{E}_{t\sim\mathrm{Unif}\{1,\cdots,T\}}
\mathbb{E}_{s_t\sim q_t(\cdot\mid s_0)}
\left[
-\sum_{i=1}^N
\big(
s_{0,i}\log\alpha_{t,i}
+
(1-s_{0,i})\log(1-\alpha_{t,i})
\big)
\right].
\end{equation*}
This objective trains $g_\phi(s_t,t)$ to approximate the clean-assortment logits from corrupted inputs.

\paragraph{Optimization-oriented reverse update.}
In a standard discrete diffusion model such as D3PM \citep{austin2021structured}, the clean-sample predictor is combined with the forward-process posterior to define a reverse transition kernel:
\begin{equation}
p_\phi(s_{t-1}\mid s_t,t)
\propto
\sum_{\tilde s_0\in\mathscr{S}}
q(s_{t-1},s_t\mid \tilde s_0)\,
\tilde p_\phi(\tilde s_0\mid s_t,t).
\label{eq:d3pmkernel}
\end{equation}
In our setting, however, the goal is not exact recovery of the reverse diffusion dynamics, but efficient generation of high-reward assortments. We therefore adopt a simpler update and directly use the predicted clean-sample distribution as the proposal for the next reverse step:
\begin{equation*}
p_\phi(s_{t-1}\mid s_t,t)
:=
\prod_{i=1}^N
\alpha_{t,i}^{s_{t-1,i}}(1-\alpha_{t,i})^{1-s_{t-1,i}}.
\end{equation*}
Equivalently,
\begin{equation*}
s_{t-1,i}\sim \mathrm{Bernoulli}(\alpha_{t,i}),
\qquad i=1,\dots,N.
\end{equation*}
Thus, the predicted clean-sample distribution itself is used as the reverse proposal. Compared with \eqref{eq:d3pmkernel}, this update is computationally cheaper and may be viewed as an accelerated denoising step, since it directly generates the next iterate without first constructing a faithful one-step reverse kernel. Empirically, this simplified update reduces computation and performs better in our reward-guided combinatorial optimization setting.

\subsection{Reward Guidance}
The denoising diffusion model described above learns a structured prior over assortments from offline data, but by itself it mainly reflects the historical behavior distribution and is not explicitly optimized for reward. To improve solution quality, we incorporate reward guidance into the reverse sampling process. The central idea is to preserve the structural prior induced by the learned denoiser while biasing generation toward assortments with higher estimated reward. This yields a guided sampling procedure that combines the exploration ability of diffusion with the optimization objective of assortment selection.
\paragraph{Guided reverse process.}
To improve upon the learned behavior prior, we incorporate reward-based guidance into the reverse diffusion dynamics. Conceptually, we bias the generative process toward high-reward assortments by modifying the reverse transition using an energy-based signal derived from the estimated reward.

Directly incorporating the global reward $R_{\hat{\theta}}(s_{t-1})$ into the reverse transition is intractable due to the high dimensionality of $s_{t-1}$. Instead, we adopt a coordinate-wise approximation based on local perturbations. For each coordinate $i \in [N]$, we consider two candidate assignments:
\begin{equation*}
s_t^{(i=1)} := (s_{t,1}, \dots, 1, \dots, s_{t,N}),
\qquad
s_t^{(i=0)} := (s_{t,1}, \dots, 0, \dots, s_{t,N}),
\end{equation*}
which differ only in the $i$-th entry. We define the marginal reward difference
\begin{equation*}
\Delta_i R_{\hat{\theta}}(s_t)
:=
R_{\hat{\theta}}(s_t^{(i=1)}) - R_{\hat{\theta}}(s_t^{(i=0)}),
\end{equation*}
which measures the incremental value of including item $i$ given the current context $s_{t,-i}$. We then modify the reverse logits by adding a reward-dependent bias:
\begin{equation}
\tilde{g}_\phi^{(i)}(s_t, t)
=
g_\phi^{(i)}(s_t, t)
+
\lambda_t \,\Delta_i R_{\hat{\theta}}(s_t),\label{eq:reward_guidance}
\end{equation}
where $\lambda_t \ge 0$ controls the strength of guidance. The resulting guided transition is
\begin{equation*}
p_{\mathrm{guide}}(s_{t-1,i}=1 \mid s_t, t)
=
\sigma\bigl(\tilde{g}_\phi^{(i)}(s_t, t)\bigr).
\end{equation*}
Equivalently, the guidance modifies the log-odds as
\begin{equation*}
\log \frac{
p_{\mathrm{guide}}(s_{t-1,i}=1 \mid s_t, t)
}{
p_{\mathrm{guide}}(s_{t-1,i}=0 \mid s_t, t)
}
=
\log \frac{
p_\phi(s_{t-1,i}=1 \mid s_t, t)
}{
p_\phi(s_{t-1,i}=0 \mid s_t, t)
}
+
\lambda_t \,\Delta_i R_{\hat{\theta}}(s_t).
\end{equation*}
This formulation admits an intuitive interpretation: the diffusion model provides a data-driven prior over assortments, while the reward difference $\Delta_i R_{\hat\theta}(s_t)$ acts as a local energy gradient that encourages the inclusion of items that increase reward.

\paragraph{Guidance schedule.}
The reliability of reward estimates depends on the noise level of $s_t$. When $t$ is large, $s_t$ is highly corrupted and reward differences are less informative; when $t$ is small, $s_t$ is closer to a valid assortment and reward signals are more reliable. To account for this, we introduce a time-dependent guidance strength
\begin{equation}
\lambda_t
=
\lambda_{\max}
\left(
1-\frac{t-1}{T-1}
\right)^\gamma,
\qquad t=1,\dots,T,\label{eq:scheduler}
\end{equation}
where $\lambda_{\max}>0$ is the maximum guidance strength and $\gamma \ge 1$ controls the concentration of guidance near the final denoising steps. In particular, $\lambda_T \approx 0$ suppresses guidance for highly noisy states, while $\lambda_1=\lambda_{\max}$ applies the strongest guidance near the end of the reverse process.

\paragraph{Sampling and properties.}
Starting from $s_T \sim \mathrm{Uniform}(\{0,1\}^N)$, we generate samples via the guided reverse process
\begin{equation*}
s_{t-1} \sim p_{\mathrm{guide}}(\cdot \mid s_t, t),
\qquad t = T, \dots, 1.
\end{equation*}
This procedure produces a set of candidate assortments that jointly reflect the structural prior learned from offline data and the reward-driven guidance. The proposed guided diffusion framework has several desirable properties:
\begin{itemize}[itemsep=-3pt, topsep=2pt]
\item In contrast to deterministic optimization methods that return a single solution, it generates a diverse set of high-quality assortments, providing flexibility in practical scenarios where robustness, uncertainty, or downstream constraints are important.
\item It avoids explicit combinatorial search over the $2^N$ subsets of $[N]$ by operating through local stochastic updates.
\item It enables efficient reward evaluation via coordinate-wise perturbations, requiring only $O(N)$ computations per step.
\item It naturally integrates reward-based guidance into the generative process, yielding a principled balance between exploration (from diffusion prior) and exploitation (from reward signal).
\end{itemize}

Overall, this framework transforms a challenging combinatorial optimization problem into a tractable guided sampling procedure, providing an effective and flexible mechanism for policy improvement in the offline setting.

\subsection{Theoretical Properties}
We provide several elementary theoretical properties of the proposed guided discrete diffusion procedure.
These results are not intended to establish global convergence to the optimal assortment.
Rather, they clarify the roles of the forward corruption process, the denoising objective, and the reward-guidance mechanism. First, we show that the Bernoulli corruption process approaches the uniform reference distribution at a controlled finite-time rate. Second, we characterize the population minimizer of the denoising objective as the posterior clean-sample marginal induced by the forward corruption process. Third, we show that the proposed reward-guided logit update is equivalent to a KL-regularized local policy improvement step under a coordinate-wise local reward surrogate.

We first justify the use of the uniform distribution as the terminal reference distribution for the forward process. 
The forward corruption kernel independently either preserves each bit or resamples it from a uniform Bernoulli distribution. 
Therefore, the amount of information retained from the initial assortment can be tracked explicitly through the cumulative retention factor
\[
\bar\alpha_t:=\prod_{\tau=1}^t(1-\beta_\tau).
\]
The following proposition gives a finite-time bound showing that the forward process approaches the uniform distribution as $\bar\alpha_t$ becomes small.
\begin{proposition}[Forward corruption approaches the uniform distribution]
\label{prop:forward_mixing}
Let $(s_t)_{t\geq 0}$ be the forward corruption process on $\{0,1\}^N$ defined by
\[
q_t(s_t\mid s_{t-1})
=
\prod_{i=1}^N
\left[
(1-\beta_t)\mathbf{1}\{s_{t,i}=s_{t-1,i}\}
+
\frac{\beta_t}{2}
\right],
\]
where $\beta_t\in(0,1)$. Define
\[
\bar\alpha_t
:=
\prod_{\tau=1}^t(1-\beta_\tau).
\]
Then, for any initial state $s_0\in\{0,1\}^N$ and any coordinate $i\in[N]$,
\[
\mathbb{P}(s_{t,i}=1\mid s_0)
=
\frac{1}{2}
+
\bar\alpha_t
\left(s_{0,i}-\frac{1}{2}\right).
\]
Equivalently, the conditional law of $s_t$ given $s_0$ factorizes as
\[
q_t(s_t\mid s_0)
=
\prod_{i=1}^N
\left[
\frac{1}{2}
+
\bar\alpha_t
\left(s_{0,i}-\frac{1}{2}\right)(2s_{t,i}-1)
\right].
\]
In particular, if $u$ denotes the uniform distribution on $\{0,1\}^N$, then
\[
\left\|
q_t(\cdot\mid s_0)-u
\right\|_{\mathrm{TV}}
\leq
\frac{N}{2}\bar\alpha_t.
\]
\end{proposition}
\begin{proof}
See Appendix \S\ref{sec:pf_fwmix} for a detailed proof.
\end{proof}
\begin{remark}
Proposition~\ref{prop:forward_mixing} is a finite-time statement. It shows that the terminal distribution is close to uniform whenever the cumulative retention factor $\bar\alpha_t=\prod_{\tau=1}^t(1-\beta_\tau)$ is small. In our implementation, we use a finite schedule with $T=100$ and $\beta_t$ linearly increasing from $10^{-4}$ to $0.1$, for which $\bar\alpha_T$ is small but nonzero. If one considers an infinite schedule satisfying $\sum_{t=1}^\infty \beta_t=\infty$, then $\bar\alpha_t\to 0$ since
\[
\log\bar\alpha_t
=
\sum_{\tau=1}^t\log(1-\beta_\tau)
\leq
-\sum_{\tau=1}^t\beta_\tau
\to -\infty,
\]
and hence the forward process converges to the uniform distribution in total variation.
\end{remark}
We next characterize what the denoising objective learns in the population limit.
The reverse model is trained to predict the clean assortment $s_0$ from a corrupted state $(s_t,t)$ using coordinate-wise binary cross-entropy.
The following result shows that, with infinite data and an unrestricted function class, the optimal predictor recovers the posterior marginal probability that each clean bit equals one.

\begin{proposition}[Population denoising learns posterior clean-sample marginals]
\label{prop:population_denoising}
Let $\mu_0$ be a distribution over clean assortments $s_0\in\{0,1\}^N$.
For each $t\in\{1,\cdots,T\}$, let $q_t(\cdot\mid s_0)$ denote the forward corruption kernel from $s_0$ to $s_t$.
Suppose the training data are generated by
\[
s_0\sim \mu_0,\qquad
t\sim \pi,\qquad
s_t\sim q_t(\cdot\mid s_0),
\]
where $\pi$ is a distribution over $\{1,\cdots,T\}$ with full support. Consider the population denoising objective over measurable logit functions
$g:\{0,1\}^N\times\{1,\cdots,T\}\to\mathbb{R}^N$:
\[
\mathcal{L}(g)
=
\mathbb{E}
\left[
-\sum_{i=1}^N
\left\{
s_{0,i}\log \sigma(g_i(s_t,t))
+
(1-s_{0,i})\log\left(1-\sigma(g_i(s_t,t))\right)
\right\}
\right].
\]
Then any population minimizer $g^\star$ satisfies, for every coordinate $i\in[N]$,
\[
\sigma(g_i^\star(s_t,t))
=
\mathbb{P}(s_{0,i}=1\mid s_t,t)
\]
almost surely with respect to the joint distribution of $(s_t,t)$. Equivalently, whenever
\[
0<
\mathbb{P}(s_{0,i}=1\mid s_t,t)
<1,
\]
the optimal logit is
\[
g_i^\star(s_t,t)
=
\log
\frac{
\mathbb{P}(s_{0,i}=1\mid s_t,t)
}{
\mathbb{P}(s_{0,i}=0\mid s_t,t)
}.
\]
\end{proposition}
\begin{proof}
See Appendix \S\ref{sec:pf_population_denoising} for a detailed proof.
\end{proof}

The previous result explains what the unguided denoiser learns from historical assortments. We now turn to the reward-guided update, which uses this learned reverse proposal as a behavior prior and locally improves it using the reward signal. While the logit-shift guidance rule \eqref{eq:reward_guidance} may appear ad hoc, it admits a simple variational interpretation.
The following result shows that it is the closed-form solution of a KL-regularized local improvement problem, with the unguided reverse proposal serving as the reference distribution.
\begin{theorem}[Reward guidance as KL-regularized local policy improvement]
\label{thm:kl_guidance}
Fix a reverse diffusion step $t$ and a current state $s_t \in \{0,1\}^N$.
Suppose the unguided reverse proposal is a factorized Bernoulli distribution
\[
p_\phi(x \mid s_t,t)
=
\prod_{i=1}^N
\mathrm{Bernoulli}
\left(
x_i;\sigma(g_\phi^{(i)}(s_t,t))
\right),
\qquad x\in\{0,1\}^N .
\]
Let $R:\{0,1\}^N\to\mathbb{R}$ be a reward function. For each coordinate $i\in[N]$, define the local reward difference
\[
\Delta_i R(s_t)
:=
R(s_t^{(i=1)})-R(s_t^{(i=0)}),
\]
where $s_t^{(i=1)}$ and $s_t^{(i=0)}$ denote the vectors obtained from $s_t$ by setting the $i$-th coordinate to $1$ and $0$, respectively. Define the coordinate-wise local reward surrogate
\[
\ell_{s_t}(x)
:=
\sum_{i=1}^N x_i \Delta_i R(s_t).
\]
For any $\lambda_t>0$, consider the KL-regularized local improvement problem
\[
\max_{q\in\Delta(\{0,1\}^N)}
\left\{
\mathbb{E}_{x\sim q}\left[\ell_{s_t}(x)\right]
-
\frac{1}{\lambda_t}
\mathrm{KL}
\left(
q\,\|\,p_\phi(\cdot\mid s_t,t)
\right)
\right\}.
\]
Then the unique optimizer is
\[
q_t^\star(x\mid s_t)
=
\frac{
p_\phi(x\mid s_t,t)
\exp\left(\lambda_t \ell_{s_t}(x)\right)
}{
Z_t(s_t)
},
\]
where
\[
Z_t(s_t)
=
\sum_{x\in\{0,1\}^N}
p_\phi(x\mid s_t,t)
\exp\left(\lambda_t \ell_{s_t}(x)\right).
\]
Moreover, $q_t^\star(\cdot\mid s_t)$ factorizes across coordinates, and its coordinate-wise logits satisfy
\[
\log
\frac{
q_t^\star(x_i=1\mid s_t)
}{
q_t^\star(x_i=0\mid s_t)
}
=
g_\phi^{(i)}(s_t,t)
+
\lambda_t \Delta_i R(s_t).
\]
Equivalently,
\[
q_t^\star(x_i=1\mid s_t)
=
\sigma
\left(
g_\phi^{(i)}(s_t,t)
+
\lambda_t \Delta_i R(s_t)
\right).
\]
Thus, the reward-guided reverse transition is the solution of a KL-regularized local policy improvement problem, where the unguided reverse proposal acts as the reference distribution and $\ell_{s_t}$ acts as a coordinate-wise local approximation of the reward.
\end{theorem}
\begin{proof}
See Appendix \S\ref{sec:pf_klguide} for a detailed proof.
\end{proof}
Theorem~\ref{thm:kl_guidance} provides a variational interpretation of the proposed guidance rule.
If the true reward $R$ were available, the guided transition would be the exact KL-regularized improvement of the unguided reverse proposal under the local surrogate $\ell_{s_t}^R$.
In the offline setting, D3AO implements a plug-in version of this update by replacing unknown $R$ with the model-based estimate $R_{\hat{\theta}}$.
Therefore, the practical transition
\[
\tilde g_\phi^{(i)}(s_t,t)
=
g_\phi^{(i)}(s_t,t)
+
\lambda_t\Delta_i R_{\hat{\theta}}(s_t)
\]
can be understood as an estimated local policy improvement step.

This interpretation is closely related to policy-based methods in reinforcement learning, especially KL-regularized policy improvement \citep{schulman2015trust, schulman2017proximal}. In these methods, policy updates favor actions with high reward or advantage while controlling deviation from a reference policy through a KL constraint or penalty.
Analogously, in our setting, the unguided reverse proposal serves as the reference policy, the coordinate-wise surrogate $\ell_{s_t}^R$ serves as a local reward or advantage signal, and the KL term keeps the guided transition close to the learned behavior prior.

\section{Numerical Experiments}
We evaluate the proposed diffusion-based assortment optimization framework on a range of synthetic settings designed to assess both solution quality and robustness. Our experiments are organized along three main dimensions. First, we examine the \emph{optimal ratio} achieved by different methods, comparing against parametric baselines under both correctly specified and misspecified choice models. Second, we study \emph{exact recovery} and \emph{diversity} of the generated assortments, highlighting the ability of our method to not only identify optimal solutions but also produce multiple high-quality and diverse candidates. Finally, we investigate \emph{sample efficiency} and robustness under \emph{distributional shift}, analyzing how performance varies with the size and quality of the offline dataset. Together, these experiments provide a comprehensive evaluation of the proposed approach in terms of accuracy, reliability, and practical applicability.
\subsection{Optimal Ratio}
In this subsection, we evaluate our diffusion-based approach on offline assortment optimization tasks, with a focus on both solution quality and robustness under model misspecification. In particular, we compare against classical parametric methods that follow an estimate-then-optimize pipeline, where the parameters of a choice model are first estimated from data, and the assortment is subsequently chosen by solving the resulting optimization problem under the fitted model using efficient, model-specific algorithms. In contrast, our method directly searches over assortments using guided sampling without relying on a correctly specified model. To provide a controlled comparison, we conduct experiments on synthetic datasets generated from multiple standard choice models and systematically vary problem size and data quality.

\paragraph{Setting.}
We consider the offline assortment optimization problem, where the goal is to select an assortment that maximizes the expected reward under an unknown customer choice model. In this setting, the choice model is not directly accessible; instead, we are given an offline dataset of historical interactions, consisting of assortments offered to customers and the corresponding observed choices and rewards. 

To evaluate different methods under controlled conditions, we generate synthetic datasets from three widely used choice models with increasing levels of complexity:
\begin{itemize}[itemsep=-2pt, topsep=2pt]
    \item \textbf{Multinomial Logit (MNL):} assumes homogeneous customer preferences and follows a standard softmax-based choice rule.
    \item \textbf{Markov Chain Choice Model (MCCM):} captures substitution effects through pairwise interactions among products.
    \item \textbf{Mixed Multinomial Logit (MMNL):} accounts for customer heterogeneity by introducing latent customer types.
\end{itemize}
Detailed formulations of these models are provided in Appendix~\S\ref{sec:choicemodels}.

We vary the number of products $N \in \{20,40,60,80,100\}$ to study scalability with respect to problem size. For each experimental setting, we generate offline datasets of size $n=10{,}000\ll 2^N$, which represents a medium-scale regime large enough to support reliable estimation while still reflecting the finite-data nature of offline AO. Historical assortments are generated according to a Boltzmann policy with inverse temperature parameter $\beta \in \{0.1, 1.0\}$, where smaller values of $\beta$ correspond to more exploratory (noisier) behavior. For each combination of data-generating mechanism, product size $N$, and inverse temperature $\beta$, we generate 10 independent datasets to ensure statistical reliability of the results.

\paragraph{Baseline Methods.}
We consider the following parametric baselines:
\begin{itemize}[itemsep=-2pt, topsep=2pt]
    \item \textbf{MNL-MLE}\; estimates item utilities under the multinomial logit (MNL) model via maximum likelihood estimation (MLE), and then solves the assortment optimization problem under the fitted model using efficient revenue-ordered (RO) policies~\citep{talluri2004revenue}.
    
    \item \textbf{MCCM-EM}\; estimates the parameters of the Markov chain choice model (MCCM) using the expectation-maximization (EM) algorithm~\citep{csimcsek2018expectation}, and then solves the assortment optimization problem using a polynomial-time linear programming formulation for Markov chain choice models~\citep{blanchet2016markov}.
\end{itemize}
In the MNL and MCCM settings, MNL-MLE and MCCM-EM are correctly specified and serve as \textit{oracle parametric baselines}, respectively. In contrast, under the MMNL setting, both methods are misspecified. All parametric baselines return a single assortment solution.

\paragraph{Guided sampling.}
Our approach adopts a model-agnostic optimization framework based on guided discrete diffusion. Starting from an initial noisy assortment, we iteratively apply a learned reverse process to generate candidate assortments. To incorporate reward information, we use guided reverse sampling, where the transition logits at each step are adjusted by a reward-based guidance term that biases the sampling process toward high-reward assortments.

We instantiate this framework using the neural choice models GAsN and RAsN proposed by \citet{wang2023neural}. For each instance, we generate 256 candidate assortments and evaluate their quality. Implementation details, including network architectures, diffusion schedules, and guidance parameters, are provided in Appendix~\S\ref{sec:implementation}.

\paragraph{Evaluation protocol.}
We evaluate all methods using the \emph{optimal ratio}, defined as
\[
\mathrm{Optimal\ Ratio\ of}\ s:=\frac{R(s)}{R(s^\star)},\quad s\in\mathscr{S},
\]
where $R(s)$ denotes the true expected reward of assortment $s$ under the ground-truth choice model, and $s^\star$ is the optimal assortment that maximizes $R(s)$.

For parametric baselines, we first estimate the model parameters from data and then solve the corresponding assortment optimization problem exactly under the fitted model. Since these methods return a single assortment, we report the optimal ratio of the returned solution.

For diffusion-based methods (GAsN- and RAsN-guided), each run generates a set of candidate assortments. We report summary statistics of their optimal ratios, including the maximum (best found solution), mean, median (\texttt{q50}), and 90th percentile (\texttt{q90}). These statistics characterize both the best-case performance and the overall quality of the generated samples. We also include an unguided sampling variant, which applies the reverse diffusion process without reward guidance, to isolate the effect of the guidance mechanism.

\begin{figure}[htbp]
\centering
\caption{
Optimal ratio across problem sizes $N$ under different choice models. 
Solid lines denote the mean optimal ratio, dashed lines denote the 90th percentile (\texttt{q90}), shaded regions indicate 95\% confidence intervals for the mean, and error bars indicate 95\% confidence intervals for \texttt{q90}.
}
\begin{subfigure}[ht]{\linewidth}
\centering
\includegraphics[width=0.75\linewidth]{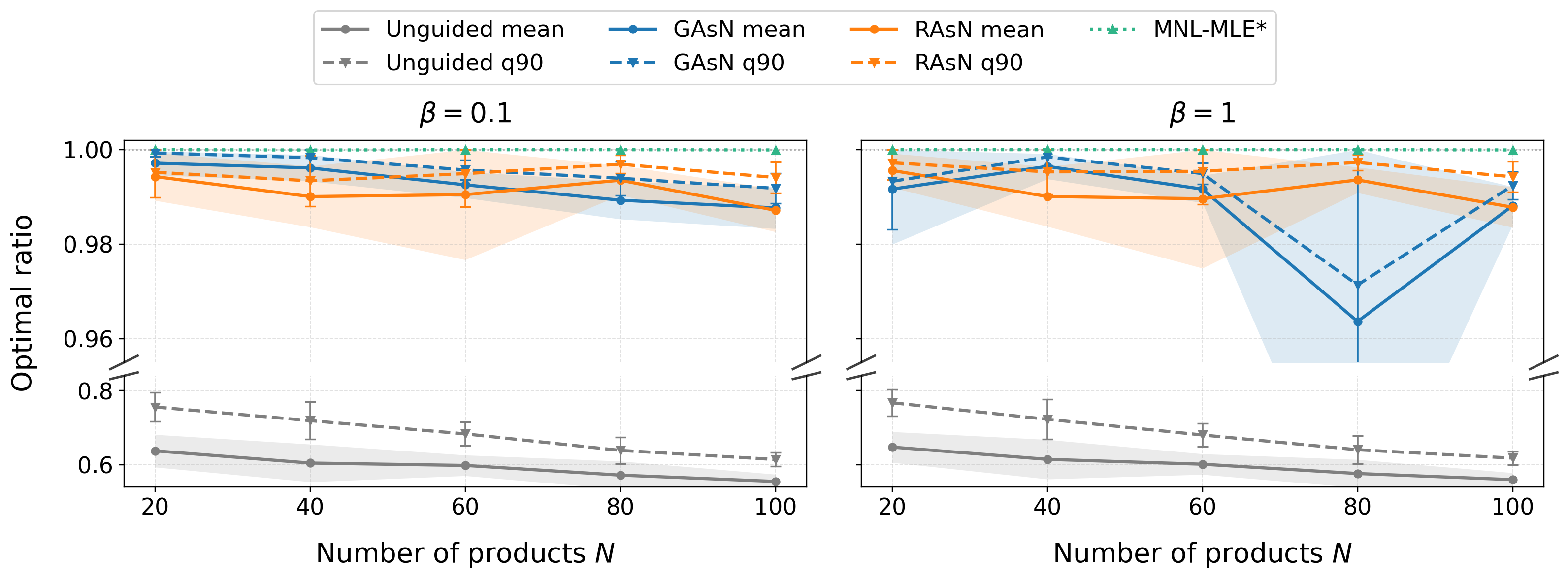}
\label{fig:mnl_results}
\caption{MNL data. The oracle baseline MNL-MLE$^{*}$ achieves near-optimal performance, while GAsN- and RAsN-guided methods remain highly competitive.}
\end{subfigure}

\vspace{0.1em}

\begin{subfigure}[ht]{\linewidth}
\centering
\includegraphics[width=0.75\linewidth]{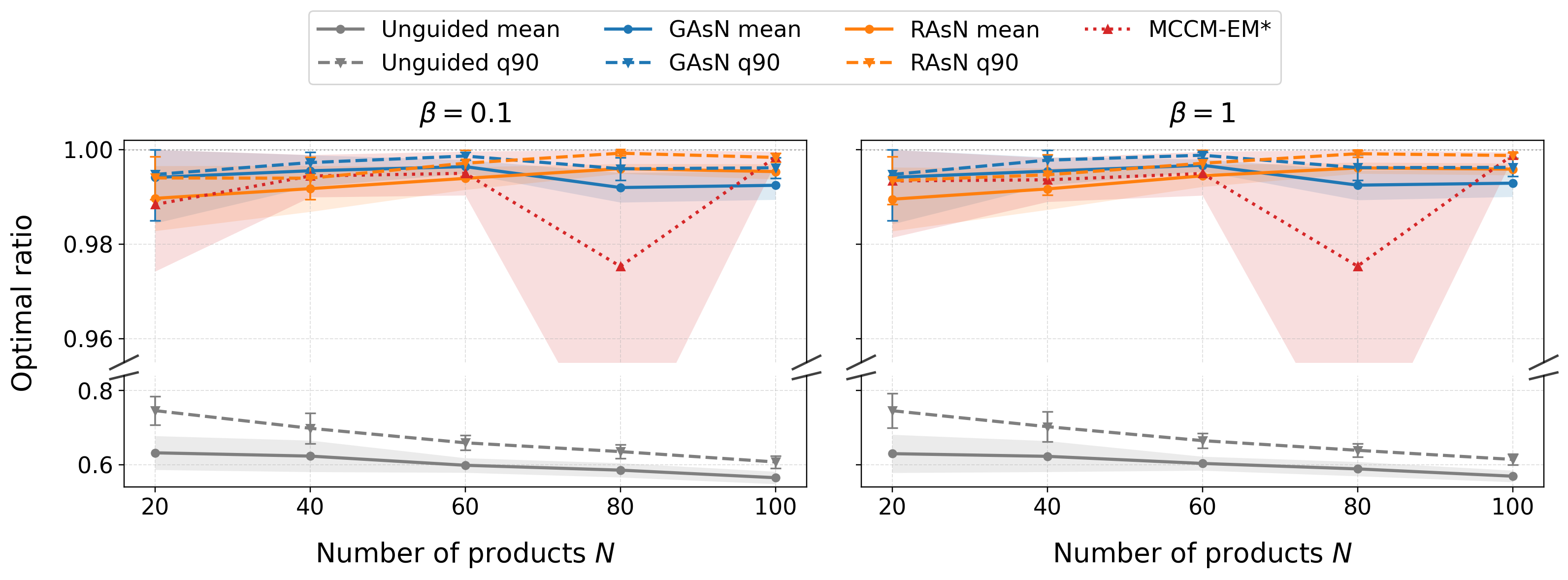}
\label{fig:mccm_results}
\caption{MCCM data. Both GAsN- and RAsN-guided methods match or outperform the oracle baseline MCCM-EM$^{*}$, while unguided sampling remains substantially suboptimal.}
\end{subfigure}

\vspace{0.1em}

\begin{subfigure}[ht]{\linewidth}
\centering
\includegraphics[width=0.75\linewidth]{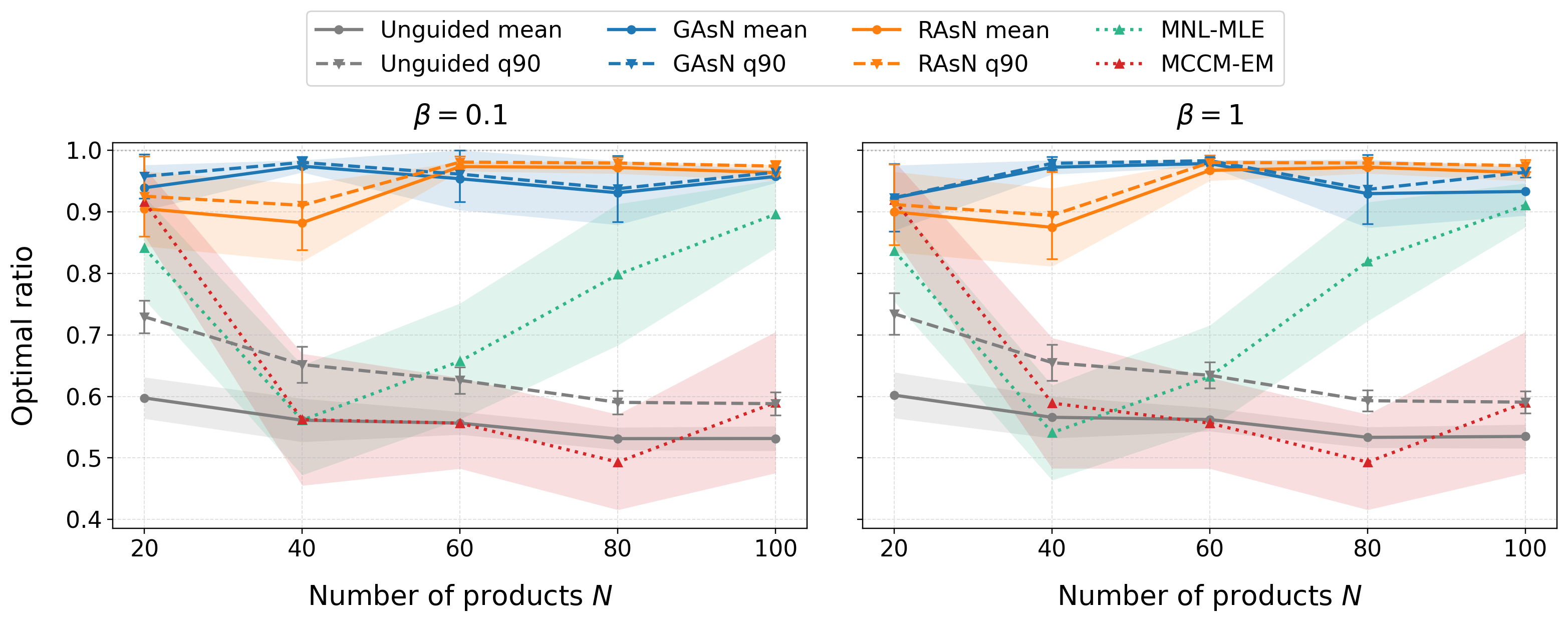}
\label{fig:mmnl_results}
\caption{MMNL data. Under model misspecification, guided diffusion methods significantly outperform parametric baselines and unguided sampling.}
\end{subfigure}
\label{fig:main_results}
\end{figure}

\paragraph{Results on MNL Data.}
Figure~\ref{fig:main_results} (a) and Table~\ref{tab:mnl_opt_ratio} report results on MNL data, where MNL-MLE serves as the correctly specified oracle baseline. Its near-perfect performance is expected, as the MNL model only requires estimating $N$ item utilities, with $N \ll n$. In our setting with $n = 10{,}000$ offline samples, these parameters can be estimated very accurately. In contrast, unguided sampling performs much worse, and both its maximum and average optimal ratios decrease as $N$ grows. Once reward guidance is introduced, both GAsN-guided and RAsN-guided improve sharply and remain near-optimal across almost all settings, with maximum optimal ratios consistently close to $100\%$. Their mean, median, and 90th percentile statistics are also uniformly high, indicating that the improvement is not limited to a few best samples. Overall, guided diffusion remains highly effective even against a correctly specified model-based baseline.

\paragraph{Results on MCCM Data.}
Figure~\ref{fig:main_results} (b) and Table~\ref{tab:mccm_opt_ratio} report results on synthetic data generated from the MCCM, where MCCM-EM is the correctly specified oracle parametric baseline. As in the MNL setting, MCCM-EM achieves a high optimal ratio across all problem sizes because it matches the true data-generating model, although its performance is slightly less perfect than in the MNL case due to the greater complexity of the MCCM. Unguided sampling again performs much worse, with both the maximum and average optimal ratios decreasing as $N$ increases. In contrast, both guided variants improve dramatically and remain close to optimal across nearly all settings. Notably, in several cases, the best found solutions slightly exceed the single solution returned by MCCM-EM. Overall, these results show that guided diffusion remains highly effective in the more structured MCCM setting.

\paragraph{Results on MMNL Data under Model Misspecification.}
Figure \ref{fig:main_results} (c) and Table~\ref{tab:mmnl_opt_ratio} report results on MMNL data, where both MNL-MLE and MCCM-EM are misspecified parametric baselines. In this setting, the model-based baselines degrade substantially and become much less stable across problem sizes, while unguided sampling also performs poorly. In contrast, both guided variants remain consistently strong and substantially outperform the misspecified baselines across all settings. This is consistent with the fact that our method does not rely on a fixed parametric assumption on the underlying choice model, and is therefore less sensitive to model mismatch. Overall, these results show that guided diffusion is substantially more robust to model misspecification and can recover high-quality assortments without requiring explicit assumptions on the data-generating model.

\subsection{Optimality and Diversity}

While the optimal-ratio results show that guided diffusion consistently generates high-reward assortments, they do not reveal whether these samples exactly recover the global optimum, nor how the sampler behaves when exact recovery becomes difficult. In relatively simple settings, a strong method should place substantial probability mass on the true optimal assortment and recover it frequently. As the problem becomes more challenging, however, exact recovery may no longer be realistic, even when the generated assortments remain near-optimal in reward. In that regime, it becomes important to ask a complementary question: whether the generative procedure still produces a sufficiently diverse set of competitive candidate solutions, rather than collapsing to a small number of repeated outputs. Motivated by this, we next evaluate guided sampling from two perspectives beyond reward quality: exact recovery of the optimal assortment and diversity among generated assortments.

\paragraph{Exact Recovery of Optimal Assortments.}
While the optimal-ratio results quantify the reward quality of generated assortments, they do not reveal whether guided sampling actually recovers the true global optimum. To address this, we evaluate the \emph{exact recovery rate} among generated samples. For each problem instance, suppose a method generates $M=256$ assortments ${s^{(1)}, \dots, s^{(M)}}$, and let $s^\star$ denote an optimal assortment under the ground-truth model. We define
\begin{equation*}
\mathrm{Rec}(S)=\frac{1}{M}\sum_{i=1}^M \mathbf{1}\{s^{(i)}=s^\star\},
\end{equation*}
that is, the proportion of generated samples that exactly match the optimal solution. Unlike reward-based metrics, exact recovery provides a stricter criterion that directly measures whether the generative procedure places sufficient probability mass on the true optimum.

\paragraph{Results on Exact Recovery.}
Table~\ref{tab:exact_prop_combined} reports the exact recovery rates under different data-generating models and problem sizes. In both the MNL and MCCM settings, the recovery rate decreases markedly as $N$ increases, indicating that exact identification of the global optimum becomes increasingly difficult as the combinatorial search space grows. Nevertheless, both guided variants retain nontrivial recovery probabilities in low- and medium-dimensional regimes. For example, on MCCM data with $\beta=1.0$, GAsN-guided achieves exact recovery rates of 68.16\% at $N=20$ and 32.73\% at $N=40$. In more challenging settings, especially under MMNL model misspecification, exact recovery becomes rare or disappears altogether. Combined with the high optimal-ratio results in Tables~\ref{tab:mnl_opt_ratio} and~\ref{tab:mccm_opt_ratio}, these findings suggest that even when exact recovery is difficult, guided diffusion continues to generate high-quality near-optimal assortments reliably.

\begin{table}[htbp]
\centering
\newcommand{\std}[1]{\ensuremath{\,{\scriptstyle \pm\, #1}}}
\caption{Proportion of exact optimal assortments. All entries are reported in percentage units (\%) and shown as mean $\pm$ standard deviation over 10 runs.}
\label{tab:exact_prop_combined}

\begin{subtable}[t]{\textwidth}
\centering
\caption{MNL data}
\begin{tabular}{clccccc}
\toprule
$\beta$ & Method & $N=20$ & $N=40$ & $N=60$ & $N=80$ & $N=100$ \\
\midrule
\multirow{2}{*}{1.0}
& GAsN-guided
& \textbf{45.20}\std{44.34}
& \textbf{30.27}\std{43.19}
& \textbf{12.42}\std{26.33}
& 0.70\std{1.41}
& 0.43\std{1.29} \\
& RAsN-guided
& 30.23\std{34.04}
& 31.02\std{35.48}
& 8.44\std{13.29}
& \textbf{7.46}\std{16.14}
& \textbf{0.55}\std{1.40} \\
\midrule
\multirow{2}{*}{0.1}
& GAsN-guided
& \textbf{44.38}\std{44.13}
& \textbf{31.17}\std{39.82}
& \textbf{13.09}\std{27.40}
& 1.25\std{3.26}
& 0.20\std{0.59} \\
& RAsN-guided
& 34.45\std{36.26}
& 28.63\std{39.36}
& 13.05\std{20.52}
& \textbf{7.73}\std{15.96}
& \textbf{0.98}\std{2.18} \\
\bottomrule
\end{tabular}
\end{subtable}

\vspace{0.8em}
\begin{subtable}[t]{\textwidth}
\centering
\caption{MCCM data}
\begin{tabular}{clccccc}
\toprule
$\beta$ & Method & $N=20$ & $N=40$ & $N=60$ & $N=80$ & $N=100$ \\
\midrule
\multirow{2}{*}{1.0}
& GAsN-guided 
& \textbf{68.16}\std{44.35} 
& \textbf{32.73}\std{33.71} 
& \textbf{5.98}\std{8.93} 
& 2.97\std{7.91} 
& 0.98\std{1.77} \\
& RAsN-guided 
& 36.45\std{44.19} 
& 26.21\std{37.13} 
& 2.73\std{3.67} 
& \textbf{7.15}\std{10.42} 
& \textbf{7.89}\std{17.11} \\
\midrule
\multirow{2}{*}{0.1}
& GAsN-guided 
& \textbf{67.77}\std{44.59} 
& \textbf{35.82}\std{38.43} 
& \textbf{6.41}\std{13.77} 
& 2.77\std{7.01} 
& 1.13\std{1.76} \\
& RAsN-guided 
& 36.76\std{44.77} 
& 27.23\std{36.83} 
& 2.97\std{4.76} 
& \textbf{6.64}\std{13.34} 
& \textbf{6.60}\std{15.46} \\
\bottomrule
\end{tabular}
\end{subtable}

\vspace{0.8em}
\begin{subtable}[t]{\textwidth}
\centering
\caption{MMNL data}
\begin{tabular}{clccccc}
\toprule
$\beta$ & Method & $N=20$ & $N=40$ & $N=60$ & $N=80$ & $N=100$ \\
\midrule
\multirow{2}{*}{1.0}
& GAsN-guided 
& 13.48\std{30.65} 
& \ 0.00\std{0.00}\ 
& \ 0.00\std{0.00}\ 
& \ 0.00\std{0.00}\ 
& \ 0.00\std{0.00}\ \\
& RAsN-guided 
& \textbf{20.16}\std{37.58} 
& \ 0.00\std{0.00}\  
& \ \textbf{3.24}\std{9.73}\  
& \ 0.00\std{0.00}\ 
& \ \textbf{1.48}\std{4.45}\ \\
\midrule
\multirow{2}{*}{0.1}
& GAsN-guided 
& 6.29\std{14.52} 
& \ \textbf{0.08}\std{0.23}\  
& \ 0.00\std{0.00}\ 
& \ 0.00\std{0.00}\ 
& \ 0.00\std{0.00}\ \\
& RAsN-guided 
& \textbf{20.12}\std{36.98} 
& \ 0.00\std{0.00}\ 
& \ \textbf{3.13}\std{9.38}\  
& \ 0.00\std{0.00}\ 
& \ \textbf{0.63}\std{1.88}\ \\
\bottomrule
\end{tabular}
\end{subtable}
\end{table}

\paragraph{Diversity of Generated Assortments.}
As shown above, exact recovery becomes increasingly difficult as the problem size grows, even when guided diffusion continues to generate near-optimal assortments. In such regimes, it is important to further understand whether the sampler still explores a nontrivial set of competitive solutions, rather than collapsing to a small number of repeated outputs. To quantify this behavior, we evaluate the diversity of generated assortments using two complementary metrics: the \emph{unique sample ratio} and the \emph{average pairwise Hamming distance}. Suppose each method generates $M$ assortments ${s^{(1)}, \dots, s^{(M)}}$, where each $s^{(i)} \in {0,1}^N$. Let $U$ denote the number of distinct assortments in this set. We define the \emph{unique sample ratio} as
\begin{equation*}
\mathrm{Uniq}(S)=\frac{U}{M},
\end{equation*}
which measures the proportion of distinct solutions among all generated samples. A higher value indicates that the generative procedure avoids collapsing to only a few repeated assortments.

To further measure structural variation among generated solutions, we consider the \emph{average pairwise Hamming distance}:
\begin{equation*}
\mathrm{Ham}(S)=\frac{2}{M(M-1)}\sum_{1\le i<j\le M}\frac{1}{N}\bigl\|s^{(i)}-s^{(j)}\bigr\|_1.
\end{equation*}
This metric measures how different two generated assortments are in terms of product inclusion, normalized by the number of products. While the unique ratio reflects \emph{support diversity}, namely how many distinct solutions are produced, the Hamming distance reflects \emph{geometric diversity}, namely how far apart these solutions are in the combinatorial space. Together, these two metrics provide a more complete picture of whether guided sampling maintains meaningful variation among high-quality candidate assortments.

\paragraph{Results on Diversity.}
Table~\ref{tab:diversity_unique_ratio} reports the diversity of generated assortments across different data-generating models and problem sizes. Overall, both the unique sample ratio and the average pairwise Hamming distance tend to increase as $N$ grows, indicating that guided sampling becomes more diverse in larger combinatorial spaces. By contrast, when $N$ is small, the generated assortments are much less diverse. This is consistent with the exact-recovery results: in simpler regimes, the optimal assortment is easier to identify, so the sampling process tends to concentrate more strongly around the optimum. As the problem becomes harder, exact recovery becomes less likely, and the sampler correspondingly spreads its probability mass over a broader set of competitive candidate solutions rather than repeatedly producing the same assortment.

\begin{table}[htbp]
\centering
\newcommand{\std}[1]{\ensuremath{\!{\scriptstyle \pm\, #1}}}
\caption{Diversity of generated assortments measured by unique sample ratio and average pairwise Hamming distance. All entries are reported in percentage units (\%) and shown as mean $\pm$ standard deviation over 10 runs, based on 256 generated assortments per run.}
\label{tab:diversity_unique_ratio}
\begin{subtable}[t]{\textwidth}
\centering
\caption{Unique ratio}
\resizebox{0.9\textwidth}{!}{%
\begin{tabular}{cllccccc}
\toprule
$\beta$ & Data model & Method & $N=20$ & $N=40$ & $N=60$ & $N=80$ & $N=100$ \\
\midrule
\multirow{6}{*}{1.0} & \multirow{2}{*}{MNL}
& GAsN
& \textbf{1.21} \std{1.07}
& \textbf{5.31} \std{4.67}
& \textbf{14.88} \std{11.78}
& 25.00 \std{26.20}
& 28.63 \std{13.82} \\
&& RAsN
& \textbf{1.21} \std{0.56}
& 3.71 \std{3.97}
& 14.38 \std{12.72}
& \textbf{27.19} \std{23.93}
& \textbf{50.74} \std{28.23} \\
\cmidrule(lr){2-8}
&\multirow{2}{*}{MCCM}
& GAsN
& 0.78 \std{0.55}
& \textbf{2.81} \std{1.59}
& 9.73 \std{4.02}
& \textbf{17.73} \std{10.18}
& 25.74 \std{14.73} \\
&& RAsN
& \textbf{1.09} \std{0.67}
& 2.15 \std{1.12}
& \textbf{12.11} \std{16.51}
& 13.44 \std{11.88}
& \textbf{27.42} \std{20.25} \\
\cmidrule(lr){2-8}
&\multirow{2}{*}{MMNL}
& GAsN
& \textbf{0.82} \std{0.41}
& 1.29 \std{0.63}
& 6.64 \std{9.54}
& \textbf{18.59} \std{16.13}
& \textbf{44.77} \std{27.14} \\
&& RAsN
& \textbf{0.82} \std{0.32}
& \textbf{3.63} \std{4.21}
& \textbf{6.68} \std{9.03}
& 16.64 \std{15.58}
& 38.75 \std{27.39} \\
\midrule
\multirow{6}{*}{0.1} & \multirow{2}{*}{MNL}
& GAsN
& \textbf{1.09} \std{0.80}
& \textbf{5.27} \std{4.22}
& 10.98 \std{5.97}
& 16.17 \std{9.30}
& 27.77 \std{13.38} \\
&& RAsN
& 0.98 \std{0.44}
& 3.59 \std{3.99}
& \textbf{12.03} \std{11.26}
& \textbf{24.06} \std{21.44}
& \textbf{52.30} \std{27.45} \\
\cmidrule(lr){2-8}
&\multirow{2}{*}{MCCM}
& GAsN
& 0.74 \std{0.37}
& \textbf{2.93} \std{1.65}
& 9.38 \std{4.90}
& \textbf{19.30} \std{10.26}
& \textbf{29.38} \std{16.42} \\
&& RAsN
& \textbf{0.90} \std{0.50}
& 1.84 \std{1.05}
& \textbf{12.42} \std{18.19}
& 13.75 \std{12.46}
& 27.30 \std{19.85} \\
\cmidrule(lr){2-8}
&\multirow{2}{*}{MMNL}
& GAsN
& \textbf{0.86} \std{0.38}
& 1.99 \std{1.00}
& \textbf{13.83} \std{22.48}
& \textbf{22.89} \std{23.31}
& 36.13 \std{26.08} \\
&& RAsN
& \textbf{0.86} \std{0.34}
& \textbf{3.12} \std{3.27}
& 4.96 \std{7.67}
& 15.66 \std{15.57}
& \textbf{40.31} \std{27.90} \\
\bottomrule
\end{tabular}%
}
\end{subtable}

\vspace{1.5em}

\begin{subtable}[t]{\textwidth}
\centering
\caption{Average pairwise Hamming distance}
\resizebox{0.9\textwidth}{!}{%
\begin{tabular}{cllccccc}
\toprule
$\beta$ & Data model & Method & $N=20$ & $N=40$ & $N=60$ & $N=80$ & $N=100$ \\
\midrule
\multirow{6}{*}{1.0}&\multirow{2}{*}{MNL}
& GAsN
& \textbf{1.56} \std{1.66}
& \textbf{2.27} \std{1.49}
& 2.73 \std{1.33}
& 2.98 \std{2.55}
& 2.44 \std{0.64} \\
&& RAsN
& 1.44 \std{0.93}
& 1.88 \std{1.95}
& \textbf{2.83} \std{1.41}
& \textbf{3.10} \std{1.73}
& \textbf{4.28} \std{2.28} \\
\cmidrule(lr){2-8}
&\multirow{2}{*}{MCCM}
& GAsN
& 0.50 \std{0.66}
& \textbf{1.51} \std{0.99}
& 2.24 \std{0.80}
& \textbf{2.19} \std{0.93}
& 2.13 \std{0.84} \\
&& RAsN
& \textbf{1.50} \std{1.83}
& 1.37 \std{0.99}
& \textbf{2.31} \std{1.84}
& 1.93 \std{1.01}
& \textbf{2.25} \std{1.04} \\
\cmidrule(lr){2-8}
&\multirow{2}{*}{MMNL}
& GAsN
& 0.62 \std{0.85}
& 0.61 \std{0.70}
& 1.53 \std{1.39}
& 2.22 \std{1.21}
& \textbf{5.65} \std{5.84} \\
&& RAsN
& \textbf{1.42} \std{2.11}
& \textbf{3.10} \std{4.28}
& \textbf{2.28} \std{3.47}
& \textbf{2.40} \std{1.22}
& 3.59 \std{2.05} \\
\midrule
\multirow{6}{*}{0.1} & \multirow{2}{*}{MNL}
& GAsN
& 1.21 \std{1.59}
& \textbf{2.19} \std{1.26}
& \textbf{2.25} \std{0.76}
& 2.14 \std{0.86}
& 2.36 \std{0.63} \\
&& RAsN
& \textbf{1.22} \std{1.09}
& 1.51 \std{1.77}
& 2.24 \std{1.31}
& \textbf{2.93} \std{1.63}
& \textbf{4.44} \std{2.29} \\
\cmidrule(lr){2-8}
&\multirow{2}{*}{MCCM}
& GAsN
& 0.63 \std{0.81}
& \textbf{1.34} \std{0.94}
& 2.10 \std{0.92}
& \textbf{2.40} \std{0.87}
& \textbf{2.30} \std{0.92} \\
&& RAsN
& \textbf{1.29} \std{1.54}
& 1.00 \std{0.81}
& \textbf{2.29} \std{1.79}
& 1.98 \std{1.23}
& 2.29 \std{1.07} \\
\cmidrule(lr){2-8}
&\multirow{2}{*}{MMNL}
& GAsN
& \textbf{1.25} \std{0.92}
& 0.78 \std{0.79}
& \textbf{2.77} \std{4.29}
& \textbf{2.76} \std{2.29}
& 2.76 \std{1.46} \\
&& RAsN
& 1.19 \std{2.16}
& \textbf{2.34} \std{2.90}
& 1.39 \std{1.14}
& 2.24 \std{1.10}
& \textbf{3.65} \std{2.05} \\
\bottomrule
\end{tabular}
}
\end{subtable}
\end{table}

\subsection{Sample Efficiency and Distribution Shift}

We further evaluate the proposed approach under two additional settings. First, we study how performance changes across different sample sizes. Second, we consider a distribution-shift setting in which the historical assortments are generated from a different policy. Together, these experiments provide a broader view of the robustness of guided diffusion under varying data conditions.

\paragraph{Sample-size setting.}
To evaluate sample efficiency, we vary the size of the offline dataset while keeping the underlying assortment optimization problem fixed. Specifically, for each setting, we generate historical data under the same choice model and compare the performance of different methods across a range of sample sizes, with
\begin{equation*}
n \in \{100, 200, 500, 1000, 2000, 5000, 10000, 20000, 50000\}.
\end{equation*}
As in the main experiments, we consider both MNL and MMNL data, and report results for two problem sizes, $N=40$ and $N=80$. For each configuration, all results are averaged over $10$ independent runs, and we report both the mean and variability across runs. For sample efficiency, performance is measured in terms of \textit{Regret}, defined as
\begin{equation*}
\text{Regret} = 1 - \text{Optimal Ratio},
\end{equation*}
so that lower values indicate solutions closer to the true optimum. Figure~\ref{fig:sample_efficiency_all} summarizes how regret changes with the amount of available offline data for the diffusion-based methods and the parametric baseline.

\paragraph{Results on sample size.}
Figure~\ref{fig:sample_efficiency_all}, Table~\ref{tab:sample_efficiency_mnl} and Table~\ref{tab:sample_efficiency_mmnl} show that the effect of sample size differs substantially across data-generating models. In the MNL setting, regret is uniformly low and generally decreases with $n$, while the correctly specified MNL-MLE baseline quickly approaches zero regret. In contrast, under MMNL data, the misspecified MNL-MLE baseline exhibits substantially larger regret across all sample sizes. The diffusion-based methods remain markedly stronger in this setting, but their performance does not improve monotonically with more data. Overall, these results suggest that the main advantage of guided diffusion lies in its robustness under model misspecification across different data regimes, rather than in a uniformly monotonic improvement with increasing sample size.

\begin{figure}[t]
\centering
\begin{subfigure}[t]{0.48\textwidth}
    \centering
    \includegraphics[width=\linewidth]{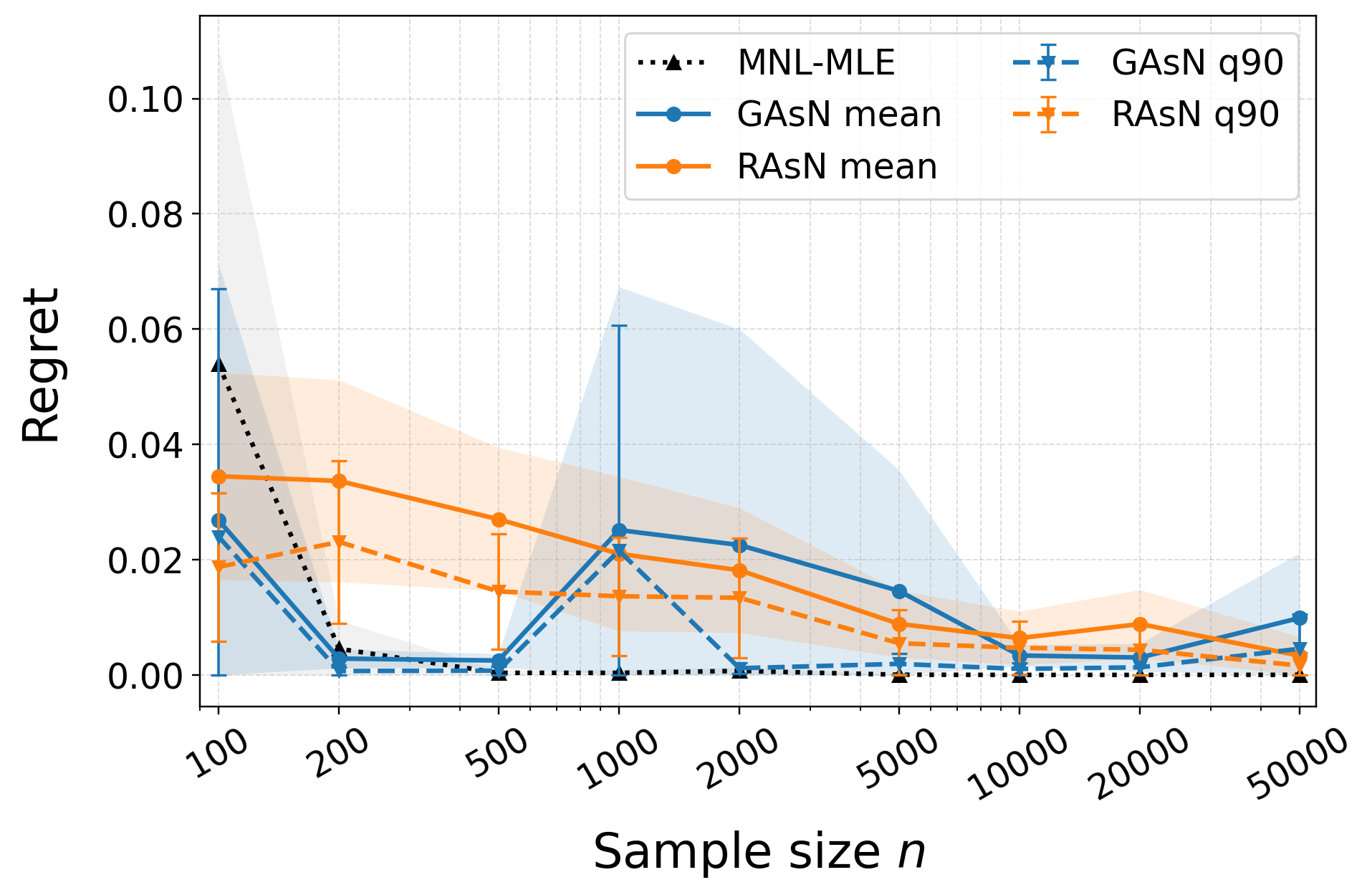}
    \caption{MNL, $N=40$}
\end{subfigure}
\hfill
\begin{subfigure}[t]{0.48\textwidth}
    \centering
    \includegraphics[width=\linewidth]{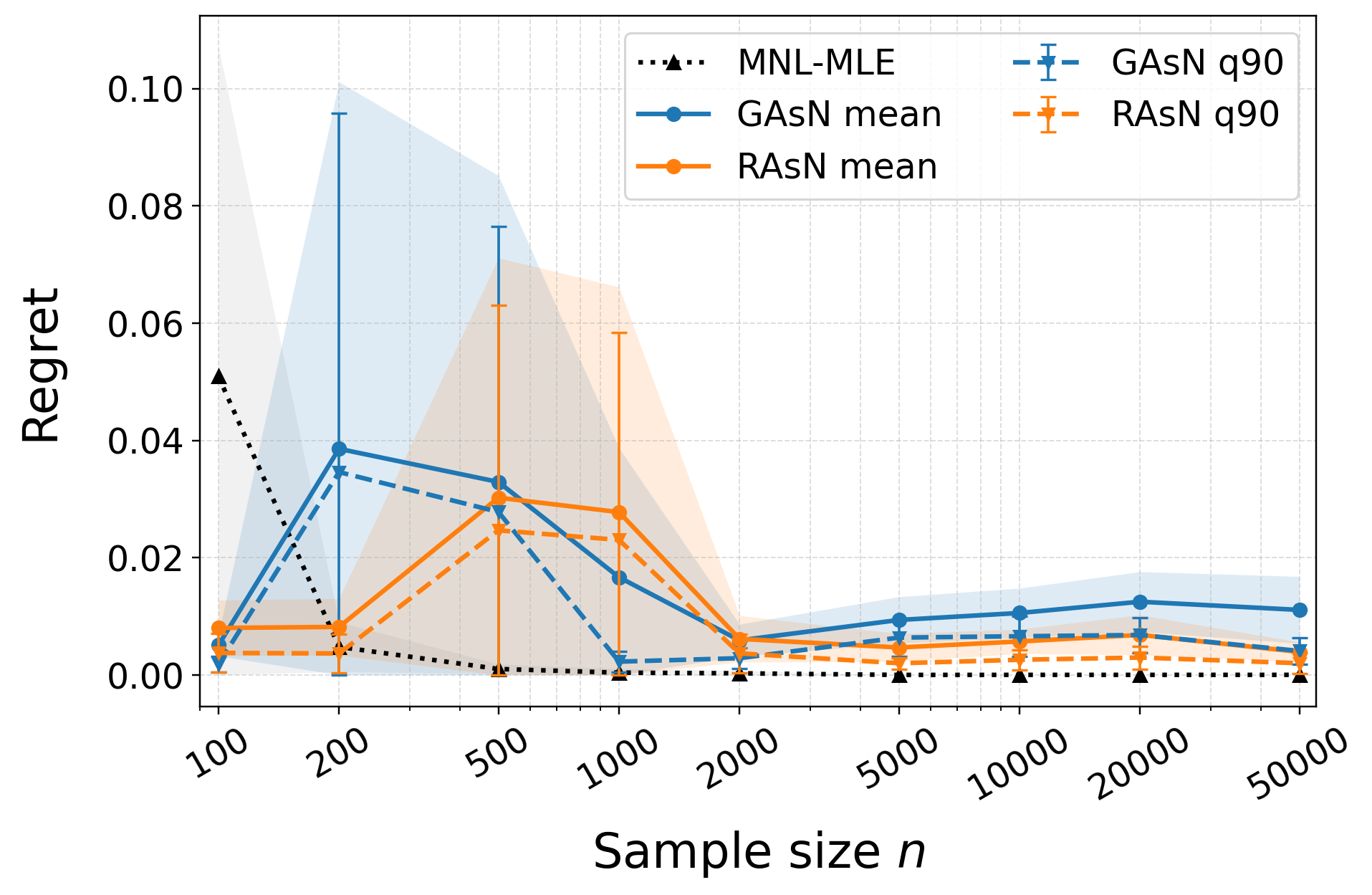}
    \caption{MNL, $N=80$}
\end{subfigure}

\vspace{0.5em}

\begin{subfigure}[t]{0.48\textwidth}
    \centering
    \includegraphics[width=\linewidth]{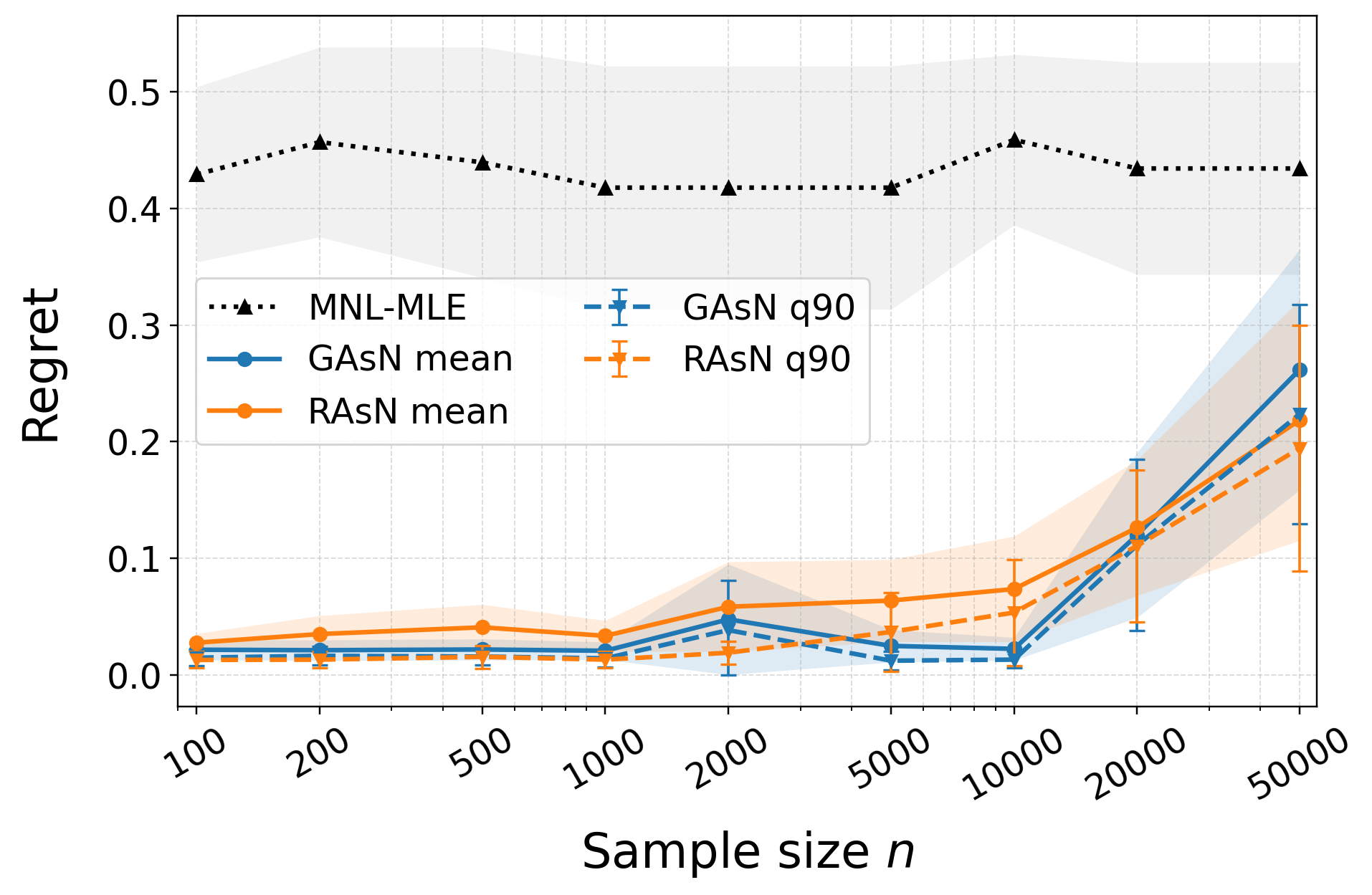}
    \caption{MMNL, $N=40$}
\end{subfigure}
\hfill
\begin{subfigure}[t]{0.48\textwidth}
    \centering
    \includegraphics[width=\linewidth]{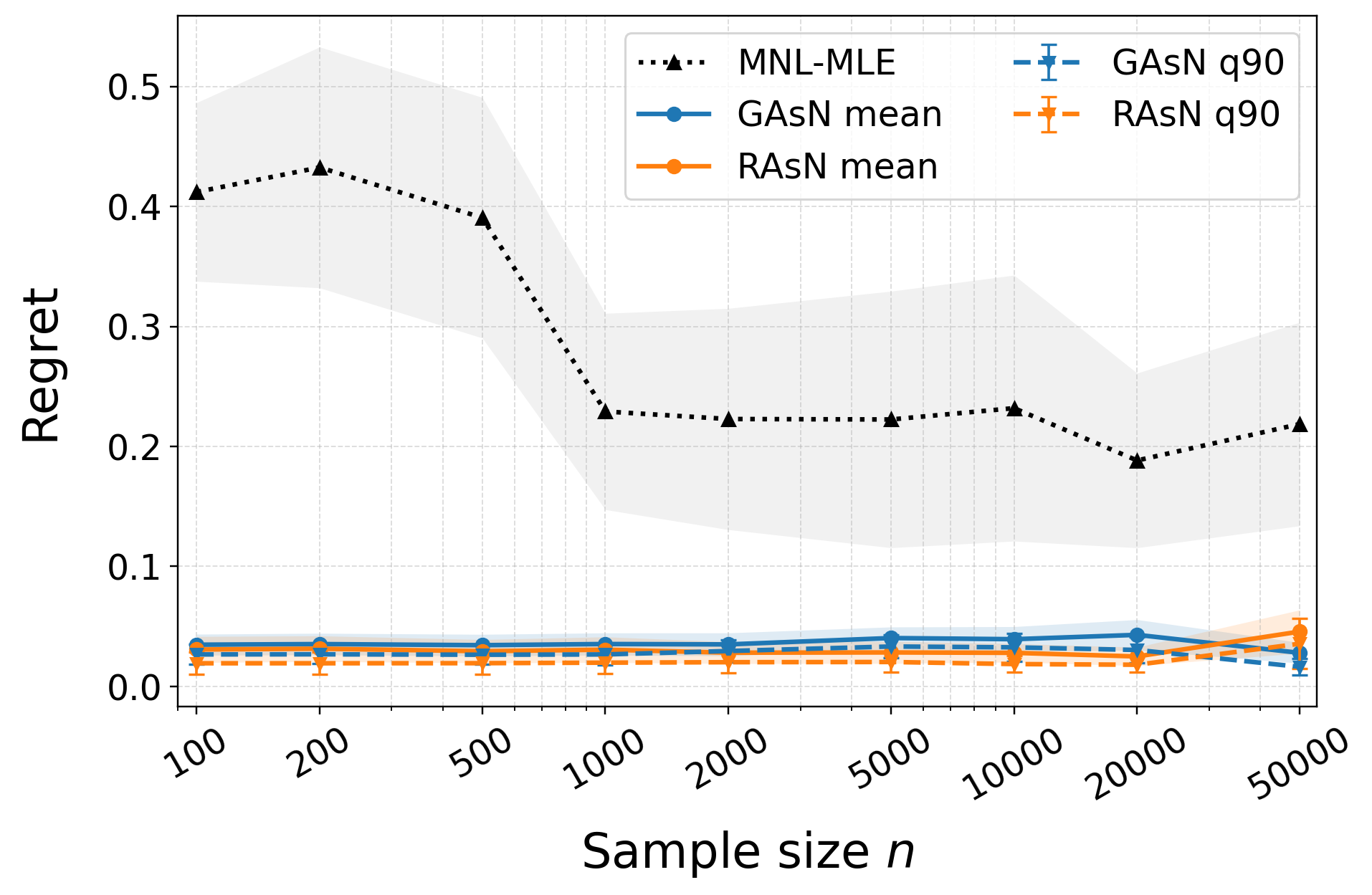}
    \caption{MMNL, $N=80$}
\end{subfigure}

\caption{Suboptimality (regret) versus sample size $n$ under MNL and MMNL choice models. Diffusion-based methods (GAsN, RAsN) achieve consistently lower regret compared to the MNL-MLE baseline, especially under model misspecification (MMNL).}
\label{fig:sample_efficiency_all}
\end{figure}

\paragraph{Distribution-shift setting.}
To further evaluate robustness beyond the data-collection setting used in the main experiments, we consider a distribution-shift setting in which historical assortments are generated from a different sampling policy. Specifically, instead of the Boltzmann policy, assortments are sampled according to a \texttt{uniform\_size\_then\_subset} procedure: for each sample, we first draw the assortment size $k$ uniformly from $[N]$, and then uniformly sample one subset of size $k$ from the combinatorial space. This produces an assortment distribution that is not Boltzmann. Given each sampled assortment, the observed choice is then generated according to the selected ground-truth choice model. Under MNL, choices are sampled from the standard softmax distribution over the offered products and the outside option. Under MMNL, choices are sampled from a mixture of type-specific MNL models with uniformly weighted latent customer types. In this way, the underlying choice mechanism is kept unchanged, while the distribution of observed assortments is shifted.

\paragraph{Results under distribution shift.}
Table~\ref{tab:dist_shift_opt_ratio} reports the optimal-ratio statistics under this shifted assortment-generation distribution. Compared with the main experiments under the Boltzmann distribution, performance drops noticeably in both the MNL and MMNL settings, especially in the mean, median, and upper-quantile statistics. Nevertheless, the maximum optimal ratios remain relatively high, indicating that guided diffusion can still recover highly competitive solutions under distribution shift. Overall, these results suggest that although the proposed approach is affected by changes in the assortment distribution, it remains reasonably robust, particularly in preserving strong best-case solution quality.

\begin{table}[ht]
\centering
\caption{Optimal ratio statistics of generated assortments under a distribution-shift setting. Historical assortments are generated from a policy different from that used in the main experiments, while observed choices are still sampled from the ground-truth choice model. All entries are reported in percentage units (\%) and shown as mean $\pm$ standard deviation over 10 runs. For each diffusion-based method, we report summary statistics over 256 generated assortments per run.}
\label{tab:dist_shift_opt_ratio}
\newcommand{\std}[1]{\ensuremath{\!{\scriptstyle \pm\, #1}}}

\resizebox{0.8\textwidth}{!}{%
\begin{tabular}{llcccc}
\toprule
\multirow{2}{*}{Method} & \multirow{2}{*}{Statistic\ } & \multicolumn{2}{c}{MNL} & \multicolumn{2}{c}{MMNL} \\
\cmidrule(r){3-4} \cmidrule(l){5-6}
 &  & $N=40$ & $N=80$ & $N=40$ & $N=80$ \\
\midrule

\multirow{4}{*}{Unguided}
& max
& 80.06 \std{6.09} & 88.91 \std{3.73}
& 65.47 \std{5.71} & 69.68 \std{7.34} \\
& mean
& 57.70 \std{0.98} & 55.68 \std{0.39}
& 38.76 \std{0.91} & 46.57 \std{0.52} \\
& \texttt{q50}
& 63.24 \std{0.07} & 55.96 \std{0.14}
& 46.10 \std{0.18} & 49.13 \std{0.08} \\
& \texttt{q90}
& 64.43 \std{0.18} & 68.21 \std{1.25}
& 48.92 \std{0.98} & 51.10 \std{0.18} \\
\midrule

\multirow{4}{*}{GAsN-guided}
& max
& 97.60 \std{1.18} & 99.17 \std{0.16}
& 97.89 \std{1.62} & \textbf{91.34} \std{7.95} \\
& mean
& 66.93 \std{0.25} & \textbf{67.28} \std{0.75}
& 60.57 \std{0.78} & \textbf{55.18} \std{0.30} \\
& \texttt{q50}
& 65.08 \std{0.14} & \textbf{63.28} \std{0.18}
& 57.51 \std{0.82} & \textbf{54.23} \std{0.29} \\
& \texttt{q90}
& 69.87 \std{0.72} & \textbf{91.12} \std{10.95}
& 71.68 \std{9.81} & \textbf{60.45} \std{0.85} \\
\midrule

\multirow{4}{*}{RAsN-guided}
& max
& \textbf{97.79} \std{0.50} & \textbf{99.39} \std{0.34}
& \textbf{98.30} \std{1.77} & 80.36 \std{7.74} \\
& mean
& \textbf{67.92} \std{0.25} & 65.07 \std{0.53}
& \textbf{66.05} \std{0.63} & 53.98 \std{0.27} \\
& \texttt{q50}
& \textbf{65.95} \std{0.25} & 62.95 \std{0.20}
& \textbf{62.96} \std{0.22} & 53.20 \std{0.31} \\
& \texttt{q90}
& \textbf{72.26} \std{0.78} & 67.26 \std{0.84}
& \textbf{78.87} \std{6.78} & 58.38 \std{0.64} \\
\bottomrule

\end{tabular}%
}
\end{table}

% \paragraph{Comparison with model-based baseline.}
% We include MCCM-EM as a baseline, which fits the ground-truth choice model under correct specification. In contrast, our method is model-agnostic and does not assume knowledge of the underlying choice model. Therefore, this experiment can be viewed as a \textit{model misspecification} setting for our approach. Despite this disadvantage, our method achieves competitive performance and often matches or exceeds MCCM-EM across different problem sizes. This suggests that the proposed diffusion-based optimization is robust to model misspecification and can effectively recover near-optimal assortments without relying on explicit parametric assumptions.

\section{Discussion and Conclusion}

In this paper, we proposed a model-agnostic framework for assortment optimization based on guided discrete diffusion. 
By treating assortments as binary vectors and performing stochastic search through a learned reverse process, our approach avoids explicit combinatorial enumeration and does not rely on restrictive parametric choice models. 
The introduction of reward-guided transitions enables the method to incorporate decision objectives directly into the generative process. 
Empirically, we show that this framework consistently produces high-quality assortments and remains effective even under model misspecification. 
Moreover, the generative nature of diffusion naturally yields a diverse set of near-optimal solutions, which is particularly useful in practical decision-making settings.

More broadly, this work highlights the potential of generative modeling as a tool for combinatorial optimization. 
Compared to classical approaches, our method offers greater flexibility and robustness, while maintaining scalability in high-dimensional settings. 
At the same time, its performance depends on the structure of the offline data and the quality of reward estimation; in particular, when the data distribution lacks informative structure (e.g., near-uniform coverage over assortments), the learned generative prior may provide limited guidance. 
Future work includes developing methods that reduce this dependence on data regularity, extending the framework to more general constraints (e.g., $As \leq b$), and exploring connections with continuous relaxation approaches such as Langevin dynamics.

%%%%%%%%%%%%-- reference --%%%%%%%%%%%
\newpage
\bibliographystyle{ims}
\bibliography{reference}

@article{dong2025pasta,
  title={PASTA: A Unified Framework for Offline Assortment Learning},
  author={Dong, Juncheng and Mo, Weibin and Qi, Zhengling and Shi, Cong and Fang, Ethan X and Tarokh, Vahid},
  journal={arXiv preprint arXiv:2510.01693},
  year={2025}
}

@article{wang2023neural,
  title={A neural network based choice model for assortment optimization},
  author={Wang, Hanzhao and Cai, Zhongze and Li, Xiaocheng and Talluri, Kalyan},
  journal={arXiv preprint arXiv:2308.05617},
  year={2023}
}

@article{han2025learning,
  title={Learning an optimal assortment policy under observational data},
  author={Han, Yuxuan and Zhong, Han and Lu, Miao and Blanchet, Jose and Zhou, Zhengyuan},
  journal={arXiv preprint arXiv:2502.06777},
  year={2025}
}

@article{austin2021structured,
  title={Structured denoising diffusion models in discrete state-spaces},
  author={Austin, Jacob and Johnson, Daniel D and Ho, Jonathan and Tarlow, Daniel and Van Den Berg, Rianne},
  journal={Advances in neural information processing systems},
  volume={34},
  pages={17981--17993},
  year={2021}
}

@inproceedings{ziebart2008maximum,
  title={Maximum entropy inverse reinforcement learning.},
  author={Ziebart, Brian D and Maas, Andrew L and Bagnell, J Andrew and Dey, Anind K and others},
  booktitle={Aaai},
  volume={8},
  pages={1433--1438},
  year={2008},
  organization={Chicago, IL, USA}
}

@article{wulfmeier2015maximum,
  title={Maximum entropy deep inverse reinforcement learning},
  author={Wulfmeier, Markus and Ondruska, Peter and Posner, Ingmar},
  journal={arXiv preprint arXiv:1507.04888},
  year={2015}
}

@inproceedings{snoswell2020revisiting,
  title={Revisiting maximum entropy inverse reinforcement learning: New perspectives and algorithms},
  author={Snoswell, Aaron J and Singh, Surya PN and Ye, Nan},
  booktitle={2020 IEEE Symposium Series on Computational Intelligence (SSCI)},
  pages={241--249},
  year={2020},
  organization={IEEE}
}

@article{qi2020data,
  title={Data-driven research in retail operations—a review},
  author={Qi, Meng and Mak, Ho-Yin and Shen, Zuo-Jun Max},
  journal={Naval Research Logistics (NRL)},
  volume={67},
  number={8},
  pages={595--616},
  year={2020},
  publisher={Wiley Online Library}
}

@article{rooderkerk2013optimizing,
  title={Optimizing retail assortments},
  author={Rooderkerk, Robert P and Van Heerde, Harald J and Bijmolt, Tammo HA},
  journal={Marketing Science},
  volume={32},
  number={5},
  pages={699--715},
  year={2013},
  publisher={INFORMS}
}

@article{jagabathula2014assortment,
  title={Assortment optimization under general choice},
  author={Jagabathula, Srikanth},
  journal={Available at SSRN 2512831},
  year={2014}
}

@article{ho2020denoising,
  title={Denoising diffusion probabilistic models},
  author={Ho, Jonathan and Jain, Ajay and Abbeel, Pieter},
  journal={Advances in neural information processing systems},
  volume={33},
  pages={6840--6851},
  year={2020}
}

@article{song2020denoising,
  title={Denoising diffusion implicit models},
  author={Song, Jiaming and Meng, Chenlin and Ermon, Stefano},
  journal={arXiv preprint arXiv:2010.02502},
  year={2020}
}

@article{song2020score,
  title={Score-based generative modeling through stochastic differential equations},
  author={Song, Yang and Sohl-Dickstein, Jascha and Kingma, Diederik P and Kumar, Abhishek and Ermon, Stefano and Poole, Ben},
  journal={arXiv preprint arXiv:2011.13456},
  year={2020}
}

@article{sun2022score,
  title={Score-based continuous-time discrete diffusion models},
  author={Sun, Haoran and Yu, Lijun and Dai, Bo and Schuurmans, Dale and Dai, Hanjun},
  journal={arXiv preprint arXiv:2211.16750},
  year={2022}
}

@inproceedings{schiff2025simple,
  title={Simple guidance mechanisms for discrete diffusion models},
  author={Schiff, Yair and Sahoo, Subham Sekhar and Phung, Hao and Wang, Guanghan and Rush, Alexander and Kuleshov, Volodymyr and Dalla-Torre, Hugo and Boshar, Sam and de Almeida, Bernardo P and Pierrot, Thomas},
  booktitle={International Conference on Learning Representations},
  volume={2025},
  pages={44153},
  year={2025}
}

@inproceedings{bansal2023universal,
  title={Universal guidance for diffusion models},
  author={Bansal, Arpit and Chu, Hong-Min and Schwarzschild, Avi and Sengupta, Soumyadip and Goldblum, Micah and Geiping, Jonas and Goldstein, Tom},
  booktitle={Proceedings of the IEEE/CVF conference on computer vision and pattern recognition},
  pages={843--852},
  year={2023}
}

@article{aloui2026score,
  title={Score-based Metropolis-Hastings for Fractional Langevin Algorithms},
  author={Aloui, Ahmed and Liao, Junyi and Hasan, Ali and Blanchet, Jose and Tarokh, Vahid},
  journal={arXiv preprint arXiv:2602.00835},
  year={2026}
}

@inproceedings{wang2025fractional,
  title={Fractional Langevin Dynamics for Combinatorial Optimization via Polynomial-Time Escape},
  author={Wang, Shiyue and Guo, Ziao and Lu, Changhong and Yan, Junchi},
  booktitle={The Thirty-ninth Annual Conference on Neural Information Processing Systems},
  year={2025}
}

@article{sun2023difusco,
  title={Difusco: Graph-based diffusion solvers for combinatorial optimization},
  author={Sun, Zhiqing and Yang, Yiming},
  journal={Advances in neural information processing systems},
  volume={36},
  pages={3706--3731},
  year={2023}
}

@article{sanokowski2024diffusion,
  title={A diffusion model framework for unsupervised neural combinatorial optimization},
  author={Sanokowski, Sebastian and Hochreiter, Sepp and Lehner, Sebastian},
  journal={arXiv preprint arXiv:2406.01661},
  year={2024}
}

@article{zhao2024disco,
  title={DISCO: Efficient diffusion solver for large-scale combinatorial optimization problems},
  author={Zhao, Hang and Yu, Kexiong and Huang, Yuhang and Yi, Renjiao and Zhu, Chenyang and Xu, Kai},
  journal={arXiv preprint arXiv:2406.19705},
  year={2024}
}

@article{mendez2014branch,
  title={A branch-and-cut algorithm for the latent-class logit assortment problem},
  author={M{\'e}ndez-D{\'\i}az, Isabel and Miranda-Bront, Juan Jos{\'e} and Vulcano, Gustavo and Zabala, Paula},
  journal={Discrete Applied Mathematics},
  volume={164},
  pages={246--263},
  year={2014},
  publisher={Elsevier}
}

@article{talluri2004revenue,
  title={Revenue management under a general discrete choice model of consumer behavior},
  author={Talluri, Kalyan and Van Ryzin, Garrett},
  journal={Management science},
  volume={50},
  number={1},
  pages={15--33},
  year={2004},
  publisher={INFORMS}
}

@article{blanchet2016markov,
  title={A Markov chain approximation to choice modeling},
  author={Blanchet, Jose and Gallego, Guillermo and Goyal, Vineet},
  journal={Operations research},
  volume={64},
  number={4},
  pages={886--905},
  year={2016},
  publisher={INFORMS}
}

@article{wang2023transformer,
  title={Transformer choice net: A transformer neural network for choice prediction},
  author={Wang, Hanzhao and Li, Xiaocheng and Talluri, Kalyan},
  journal={arXiv preprint arXiv:2310.08716},
  year={2023}
}

@article{mcfadden2000mixed,
  title={Mixed MNL models for discrete response},
  author={McFadden, Daniel and Train, Kenneth},
  journal={Journal of applied Econometrics},
  volume={15},
  number={5},
  pages={447--470},
  year={2000},
  publisher={Wiley Online Library}
}

@article{rusmevichientong2010dynamic,
  title={Dynamic assortment optimization with a multinomial logit choice model and capacity constraint},
  author={Rusmevichientong, Paat and Shen, Zuo-Jun Max and Shmoys, David B},
  journal={Operations research},
  volume={58},
  number={6},
  pages={1666--1680},
  year={2010},
  publisher={INFORMS}
}

@incollection{mcfadden1972conditional,
  title={Conditional logit analysis of qualitative choice behavior},
  booktitle = {Fontiers in {Econometrics}},
	publisher = {Academic press},
	author = {McFadden, Daniel},
	editor = {Zarembka, Paul},
	year = {1974},
	pages = {105--142}
}

@article{feng2022consumer,
  title={Consumer choice models and estimation: A review and extension},
  author={Feng, Qi and Shanthikumar, J George and Xue, Mengying},
  journal={Production and Operations Management},
  volume={31},
  number={2},
  pages={847--867},
  year={2022},
  publisher={SAGE Publications Sage CA: Los Angeles, CA}
}

@book{manski1981structural,
  title={Structural analysis of discrete data with econometric applications},
  author={Manski, Charles F and McFadden, Daniel and others},
  year={1981},
  publisher={MIT press Cambridge, MA}
}

@article{chen2024overview,
  title={An overview of diffusion models: Applications, guided generation, statistical rates and optimization},
  author={Chen, Minshuo and Mei, Song and Fan, Jianqing and Wang, Mengdi},
  journal={arXiv preprint arXiv:2404.07771},
  year={2024}
}

@article{fisher2014demand,
  title={A demand estimation procedure for retail assortment optimization with results from implementations},
  author={Fisher, Marshall and Vaidyanathan, Ramnath},
  journal={Management Science},
  volume={60},
  number={10},
  pages={2401--2415},
  year={2014},
  publisher={INFORMS}
}

@article{liu2025simultaneous,
  title={Simultaneous vs. sequential: Optimal assortment recommendation in multistore retailing},
  author={Liu, Yicheng and Chen, Xiao Alison and Liu, Yan and Wang, Zizhuo},
  journal={Manufacturing \& Service Operations Management},
  volume={27},
  number={3},
  pages={825--842},
  year={2025},
  publisher={INFORMS}
}

@article{strauss2018review,
  title={A review of choice-based revenue management: Theory and methods},
  author={Strauss, Arne K and Klein, Robert and Steinhardt, Claudius},
  journal={European journal of operational research},
  volume={271},
  number={2},
  pages={375--387},
  year={2018},
  publisher={Elsevier}
}

@article{desir2022capacitated,
  title={Capacitated assortment optimization: Hardness and approximation},
  author={D{\'e}sir, Antoine and Goyal, Vineet and Zhang, Jiawei},
  journal={Operations Research},
  volume={70},
  number={2},
  pages={893--904},
  year={2022},
  publisher={INFORMS}
}

@book{train2009discrete,
  title={Discrete choice methods with simulation},
  author={Train, Kenneth E},
  year={2009},
  publisher={Cambridge university press}
}

@article{farias2013nonparametric,
  title={A nonparametric approach to modeling choice with limited data},
  author={Farias, Vivek F and Jagabathula, Srikanth and Shah, Devavrat},
  journal={Management science},
  volume={59},
  number={2},
  pages={305--322},
  year={2013},
  publisher={INFORMS}
}

@article{bertsimas2015data,
  title={Data-driven assortment optimization},
  author={Bertsimas, Dimitris and Mi{\v{s}}ic, Velibor V},
  journal={Management Science},
  volume={1},
  pages={1--35},
  year={2015}
}

@article{csimcsek2018expectation,
  title={An expectation-maximization algorithm to estimate the parameters of the Markov chain choice model},
  author={{\c{S}}im{\c{s}}ek, A Serdar and Topaloglu, Huseyin},
  journal={Operations Research},
  volume={66},
  number={3},
  pages={748--760},
  year={2018},
  publisher={INFORMS}
}

@article{dudik2011doubly,
  title={Doubly robust policy evaluation and learning},
  author={Dud{\'\i}k, Miroslav and Langford, John and Li, Lihong},
  journal={arXiv preprint arXiv:1103.4601},
  year={2011}
}

@inproceedings{swaminathan2015counterfactual,
  title={Counterfactual risk minimization: Learning from logged bandit feedback},
  author={Swaminathan, Adith and Joachims, Thorsten},
  booktitle={International conference on machine learning},
  pages={814--823},
  year={2015},
  organization={PMLR}
}

@article{levine2020offline,
  title={Offline reinforcement learning: Tutorial, review, and perspectives on open problems},
  author={Levine, Sergey and Kumar, Aviral and Tucker, George and Fu, Justin},
  journal={arXiv preprint arXiv:2005.01643},
  year={2020}
}

@article{prudencio2023survey,
  title={A survey on offline reinforcement learning: Taxonomy, review, and open problems},
  author={Prudencio, Rafael Figueiredo and Maximo, Marcos ROA and Colombini, Esther Luna},
  journal={IEEE transactions on neural networks and learning systems},
  volume={35},
  number={8},
  pages={10237--10257},
  year={2023},
  publisher={IEEE}
}

@inproceedings{shimizu2024effective,
  title={Effective off-policy evaluation and learning in contextual combinatorial bandits},
  author={Shimizu, Tatsuhiro and Tanaka, Koichi and Kishimoto, Ren and Kiyohara, Haruka and Nomura, Masahiro and Saito, Yuta},
  booktitle={Proceedings of the 18th ACM Conference on Recommender Systems},
  pages={733--741},
  year={2024}
}

@article{saure2013optimal,
  title={Optimal dynamic assortment planning with demand learning},
  author={Saur{\'e}, Denis and Zeevi, Assaf},
  journal={Manufacturing \& Service Operations Management},
  volume={15},
  number={3},
  pages={387--404},
  year={2013},
  publisher={INFORMS}
}

@article{gong2022online,
  title={Online assortment optimization with reusable resources},
  author={Gong, Xiao-Yue and Goyal, Vineet and Iyengar, Garud N and Simchi-Levi, David and Udwani, Rajan and Wang, Shuangyu},
  journal={Management Science},
  volume={68},
  number={7},
  pages={4772--4785},
  year={2022},
  publisher={INFORMS}
}

@article{li2025online,
  title={Online learning for constrained assortment optimization under markov chain choice model},
  author={Li, Shukai and Luo, Qi and Huang, Zhiyuan and Shi, Cong},
  journal={Operations research},
  volume={73},
  number={1},
  pages={109--138},
  year={2025},
  publisher={Informs}
}

@inproceedings{fujimoto2019off,
  title={Off-policy deep reinforcement learning without exploration},
  author={Fujimoto, Scott and Meger, David and Precup, Doina},
  booktitle={International conference on machine learning},
  pages={2052--2062},
  year={2019},
  organization={PMLR}
}

@article{kumar2020conservative,
  title={Conservative q-learning for offline reinforcement learning},
  author={Kumar, Aviral and Zhou, Aurick and Tucker, George and Levine, Sergey},
  journal={Advances in neural information processing systems},
  volume={33},
  pages={1179--1191},
  year={2020}
}

@inproceedings{joachims2018deep,
  title={Deep learning with logged bandit feedback},
  author={Joachims, Thorsten and Swaminathan, Adith and De Rijke, Maarten},
  booktitle={International Conference on Learning Representations},
  year={2018}
}

@article{chen2022decision,
  title={Decision forest: A nonparametric approach to modeling irrational choice},
  author={Chen, Yi-Chun and Mi{\v{s}}i{\'c}, Velibor V},
  journal={Management Science},
  volume={68},
  number={10},
  pages={7090--7111},
  year={2022},
  publisher={INFORMS}
}

@article{jagabathula2020conditional,
  title={A conditional gradient approach for nonparametric estimation of mixing distributions},
  author={Jagabathula, Srikanth and Subramanian, Lakshminarayanan and Venkataraman, Ashwin},
  journal={Management Science},
  volume={66},
  number={8},
  pages={3635--3656},
  year={2020},
  publisher={INFORMS}
}

@article{kok2008assortment,
  title={Assortment planning: Review of literature and industry practice},
  author={K{\"o}k, A G{\"u}rhan and Fisher, Marshall L and Vaidyanathan, Ramnath},
  journal={Retail supply chain management: Quantitative models and empirical studies},
  pages={99--153},
  year={2008},
  publisher={Springer}
}

@article{sturt2025value,
  title={The value of robust assortment optimization under ranking-based choice models},
  author={Sturt, Bradley},
  journal={Management Science},
  volume={71},
  number={5},
  pages={4246--4265},
  year={2025},
  publisher={INFORMS}
}

@article{bront2009column,
  title={A column generation algorithm for choice-based network revenue management},
  author={Bront, Juan Jos{\'e} Miranda and M{\'e}ndez-D{\'\i}az, Isabel and Vulcano, Gustavo},
  journal={Operations research},
  volume={57},
  number={3},
  pages={769--784},
  year={2009},
  publisher={INFORMS}
}

@article{desir2020constrained,
  title={Constrained assortment optimization under the Markov chain--based choice model},
  author={D{\'e}sir, Antoine and Goyal, Vineet and Segev, Danny and Ye, Chun},
  journal={Management Science},
  volume={66},
  number={2},
  pages={698--721},
  year={2020},
  publisher={INFORMS}
}

@article{rusmevichientong2014assortment,
  title={Assortment optimization under the multinomial logit model with random choice parameters},
  author={Rusmevichientong, Paat and Shmoys, David and Tong, Chaoxu and Topaloglu, Huseyin},
  journal={Production and Operations Management},
  volume={23},
  number={11},
  pages={2023--2039},
  year={2014},
  publisher={SAGE Publications Sage CA: Los Angeles, CA}
}

@article{shannon1948mathematical,
  title={A mathematical theory of communication},
  author={Shannon, Claude Elwood},
  journal={The Bell system technical journal},
  volume={27},
  number={3},
  pages={379--423},
  year={1948},
  publisher={Nokia Bell Labs}
}

@inproceedings{schulman2015trust,
  title={Trust region policy optimization},
  author={Schulman, John and Levine, Sergey and Abbeel, Pieter and Jordan, Michael and Moritz, Philipp},
  booktitle={International conference on machine learning},
  pages={1889--1897},
  year={2015},
  organization={PMLR}
}

@article{schulman2017proximal,
  title={Proximal policy optimization algorithms},
  author={Schulman, John and Wolski, Filip and Dhariwal, Prafulla and Radford, Alec and Klimov, Oleg},
  journal={arXiv preprint arXiv:1707.06347},
  year={2017}
}

\newpage
\appendix
\section{Experimental Details}
\subsection{Choice Models}\label{sec:choicemodels}
In this subsection, we introduce several benchmark choice models based on different assumptions about customer behavior. Given an assortment $s\in\mathscr{S}$, each model defines a choice policy $p(\cdot \mid s)$ over $s \cup \{0\}$, where $0$ denotes the no-purchase option.

\paragraph{Multinomial Logit (MNL).}
The MNL model is a classical random-utility model \citep{mcfadden1972conditional, talluri2004revenue, train2009discrete}. For each $i \in [N] \cup \{0\}$, let
\begin{equation*}
U_i = u_i + \varepsilon_i,
\end{equation*}
where $u_i$ is the deterministic utility component and the random noise terms $\{\varepsilon_i\}$ are i.i.d.\ type-I extreme-value (Gumbel). Since only utility differences matter, one typically normalizes the outside option utility, e.g., $u_0 = 0$. Under this assumption, the choice probabilities admit the closed-form expression
\begin{equation*}
p(i \mid s)
=
\frac{\exp(u_i)}{\sum_{j \in s\,\cup\,\{0\}} \exp(u_j)},
\qquad i \in s \cup \{0\}.
\end{equation*}
A key advantage of the MNL model is that the assortment optimization problem admits an efficient solution. In particular, it is well known that an optimal assortment can be obtained via a revenue-ordered (RO) policy, which selects a prefix of products sorted by revenue \citep{talluri2004revenue}. This structural property enables polynomial-time optimization and makes MNL a widely used benchmark in assortment optimization.

\textit{Parameter generation.}
To generate synthetic MNL instances, we assign utilities $\{u_i\}_{i=0}^N$ independently. The no-purchase option is assigned zero utility, i.e., $u_0 = 0$, while product utilities are sampled as
\begin{equation*}
u_1,\cdots,u_N \overset{\mathrm{i.i.d.}}{\sim} \mathcal{N}(0,\, 1).
\end{equation*}
This construction yields a homogeneous preference structure across products, consistent with the standard MNL model.

\paragraph{Markov Chain Choice Model (MCCM).}
The Markov chain choice model is introduced by \cite{blanchet2016markov}, which defines a discrete-time Markov chain on the state space $[N] \cup \{0\}$. A customer starts from an initial state distributed according to $\lambda\in\Delta([N]\cup\{0\})$, i.e. $$\sum_{i \in [N]\,\cup\,\{0\}} \lambda_i = 1,\quad\lambda_0,\lambda_1,\cdots,\lambda_N\geq 0.$$ If the current state belongs to $s \cup \{0\}$, the process stops and that state is chosen; otherwise, the customer transitions according to a transition matrix $\rho$, where $\rho_{ij} \ge 0$ and $\sum_{j \in [N]\,\cup\,\{0\}} \rho_{ij} = 1$. Let $\{X_t\}_{t \ge 0}$ denote the chain and define the first hitting time
\begin{equation*}
\tau = \inf\{t \ge 0 : X_t \in s \cup \{0\}\}.
\end{equation*}
Then the induced choice probabilities are
\begin{equation*}
p(i \mid s)
=
\mathbb{P}(X_\tau = i),
\qquad i \in s \cup \{0\}.
\end{equation*}
Compared with MNL, the MCCM framework can capture richer substitution patterns, while the MNL model can be viewed as a special case of MCCM \citep{blanchet2016markov}. Parameter estimation for MCCM can be efficiently performed via expectation-maximization (EM) algorithms \citep{csimcsek2018expectation}. Moreover, \cite{blanchet2016markov} shows that the assortment optimization problem under MCCM admits a polynomial-time solution via a linear programming (LP)-based formulation, making it a flexible yet tractable extension of classical choice models.

\textit{Parameter generation.}
To generate synthetic MCCM instances, we randomly sample the initial distribution $\lambda$ and transition matrix $\rho$. Specifically, we sample $\lambda$ from a symmetric Dirichlet distribution:
\[
\lambda \sim \mathrm{Dirichlet}(\alpha \mathbf{1}_{N+1}),
\]
where $\alpha > 0$ is a concentration parameter.

The transition matrix $\rho$ is constructed row-wise. The no-purchase state is made absorbing, i.e., $\rho_{00} = 1$. For each product state $i \in [N]$, we sample transition probabilities over all other states using a Dirichlet distribution:
\[
(\rho_{i,0}, \rho_{i,1}, \cdots, \rho_{i,i-1}, \rho_{i,i+1}, \cdots, \rho_{i,N})
\sim \mathrm{Dirichlet}(\alpha \mathbf{1}_N),
\qquad \rho_{ii} = 0.
\]
This construction provides a flexible and unbiased way to generate valid stochastic matrices while avoiding degenerate structures. The symmetric Dirichlet distribution induces diverse transition patterns across products, enabling rich context-dependent substitution effects. In our experiments, we set $\alpha = 1$ to obtain moderately heterogeneous yet well-conditioned instances.

\paragraph{Mixed Multinomial Logit (MMNL).}
MMNL extends MNL by allowing customer-level preference heterogeneity \citep{mcfadden2000mixed, train2009discrete}. In our experiments, we use a version with a finite number of customer types, which is also standard in assortment optimization \citep{mendez2014branch}. Let $\mathcal{C}$ be the set of customer types, let $\alpha_c$ be the probability of type $c$, and let $u_{c,i}$ be the utility of item $i$ for type $c$. Then
\begin{equation*}
p(i \mid s)
=
\sum_{c \in \mathcal{C}} \alpha_c
\frac{\exp(u_{c,i})}{\sum_{j \in s\,\cup\,\{0\}} \exp(u_{c,j})},
\qquad i \in s \cup \{0\},
\end{equation*}
where $\alpha_c \ge 0$ and $\sum_{c \in \mathcal{C}} \alpha_c = 1$. In other words, the population is modeled as a mixture of finitely many MNL-type customer groups with different utility parameters.

While MMNL provides a significantly more expressive model than MNL, it also introduces substantial computational challenges. For small to moderate problem sizes, or for instances with special structure, exact or near-exact solutions can be obtained via mixed-integer or branch-and-bound methods~\citep{mendez2014branch, rusmevichientong2014assortment}. However, in general, assortment optimization under MMNL is NP-hard and lacks efficient exact solution methods~\citep{bront2009column, rusmevichientong2014assortment}.

\textit{Parameter generation.}
To generate synthetic MMNL instances, we consider a finite mixture model with $|\mathcal{C}| = 5$ customer types and uniform mixing weights $\alpha_c = 0.2$ for all $c \in \mathcal{C}$. For each type $c\in \mathcal{C}$, we specify type-dependent utilities $\{u_{c,i}\}_{i=0}^N$ as follows:
\begin{equation*}
u_{c,i} =
\begin{cases}
0, & i = 0, \\[4pt]
\mathcal{N}(c + N/5,\, 1), 
& (c-1)\tfrac{N}{5} + 1 \le i \le \tfrac{cN}{5}, \\[4pt]
\mathcal{N}(-1,\, 1), & \text{otherwise}.
\end{cases}
\end{equation*}
In other words, the no-purchase option is assigned zero utility for every customer type, and the $N$ products are partitioned into five disjoint groups of equal size. For each type $c$, products in the $c$-th group are assigned higher utilities, while all remaining products receive lower utilities.

This construction induces structured preference heterogeneity, where each customer type strongly favors a distinct subset of products while maintaining mild overlap across types. As a result, the induced choice behavior exhibits complex substitution patterns that cannot be captured by homogeneous models such as MNL.

\subsection{Neural Choice Models}
Neural choice models replace a fixed parametric substitution structure by a flexible function class learned directly from data \citep{wang2023neural, wang2023transformer}. For neural models, we encode each assortment $s \in \mathscr{S}$ by its binary incidence vector in $\{0,1\}^N$, where the $i$th entry equals $1$ if and only if item $i$ is offered. Given this representation, a neural network outputs logits for the outside option and all products, followed by a masked softmax that enforces feasibility:
\begin{equation*}
\mathbf{Y} = g(s; \theta) \in \Delta([N] \cup \{0\}),
\end{equation*}
where $Y_a = p_\theta(a\,|\,s)$ for $a \in [N] \cup \{0\}$ and $p_\theta(i\,|\, s)=0$ whenever $i \notin s$. This formulation is expressive enough to capture nonlinear interactions and complex substitution effects that are difficult to specify in closed form.

\subsubsection{Architectures and Implementation}\label{sec:implementation}
In our experiments, we mainly tested two architectures from \citet{wang2023neural}. The outside option is handled by a separate logit that is always available.

\paragraph{Gated-Assort-Net (GAsN).}
GAsN is a feed-forward network that maps the assortment vector to product logits, followed by a masking step that enforces zero probability for unavailable products. Let $\mathbf{z}_0 = s$. For hidden layers $l = 1, \dots, L$, we use
\begin{equation*}
    \mathbf{z}_l = (\mathbf{W}_l \mathbf{z}_{l-1} + \mathbf{b}_l)^+,
\end{equation*}
where $(\cdot)^+ = \max(\cdot, 0)$ is applied elementwise. At the output layer, the product scores are gated by the assortment mask through a coordinatewise multiplication, so products not contained in $s$ are removed before the final normalization. The resulting masked product scores are then combined with an outside-option logit and passed through a softmax over $[N] \cup \{0\}$. Therefore, products outside the offered assortment automatically receive probability zero.

\paragraph{Residual-Assort-Net (RAsN).}
RAsN uses the same masked output layer as GAsN, but replaces the hidden transformation by residual blocks:
\begin{equation*}
    \mathbf{z}_l = (\mathbf{W}_l \mathbf{z}_{l-1} + \mathbf{b}_l)^+ + \mathbf{z}_{l-1}.
\end{equation*}
The skip connection helps preserve information from earlier layers and stabilizes training. As in GAsN, the output layer applies a coordinatewise multiplication with the assortment mask before normalization, which retains the support constraint $p_\theta(i \mid s)=0$ for unavailable items.

\paragraph{Discrete Diffusion Model.}
The diffusion model follows the discrete denoising formulation introduced in \citet{austin2021structured} and detailed in Section~\ref{sec:diffusion}, specialized to the binary state space $\{0,1\}^N$. In our setting, the forward kernel factorizes across coordinates as
\begin{equation*}
q_t(s_t \mid s_{t-1})
=
\prod_{i=1}^N \left[(1-\beta_t)\mathbf{1}\{s_{t,i}=s_{t-1,i}\} + \frac{\beta_t}{2}\right],
\end{equation*}
while the reverse model is parameterized as
\begin{equation*}
p_\phi(s_{t-1} \mid s_t,t)
=
\prod_{i=1}^N \mathrm{Bernoulli}\!\left(\sigma\bigl(g_\phi^{(i)}(s_t,t)\bigr)\right).
\end{equation*}
This is the baseline binary diffusion structure inherited from \cite{austin2021structured}; our modification is to add the reward-guidance term to the reverse logits during sampling, as described in Section~\ref{sec:diffusion}.

\paragraph{Implementation details.}
We use GAsN with a single hidden layer of width $N$ and RAsN with a single residual block of width $N$, where $N$ is the number of products. The diffusion model uses a two-layer MLP with hidden sizes $[128,128]$ and $T=100$ diffusion steps. The forward corruption schedule is linear, with $\beta_t$ ranging from $10^{-4}$ to $0.1$. For guided sampling, we use the schedule \eqref{eq:scheduler} in Section~\ref{sec:diffusion} with $\lambda_{\max}=1000$ and $\gamma=3$.

\section{Numerical Results}

In this section, we provide additional numerical results that complement the main text. In particular, we include detailed tables reporting optimal ratio statistics across all experimental settings, which are omitted from the main paper due to space constraints. These tables present comprehensive summary statistics for both parametric baselines and diffusion-based methods, including multiple quantiles and variability measures over repeated runs. Together with the figures in the main text, these results offer a more complete view of the performance and stability of the proposed approach.

\begin{table}[htbp]
\centering
\caption{Optimal ratio statistics of generated assortments on MNL data. All entries are reported in percentage units (\%) and shown as mean $\pm$ standard deviation over 10 runs. The ground-truth choice model is MNL, and MNL-MLE is the correctly specified parametric baseline. For this baseline, we report the quality of the single solution returned by the method; for GAsN and RAsN-guided methods, we report summary statistics over generated 256 assortments. * marks the oracle parametric model (i.e., the model class matches the ground-truth choice model).}
\label{tab:mnl_opt_ratio}
\newcommand{\std}[1]{\ensuremath{\!{\scriptstyle \pm\, #1}}}
\begin{subtable}[t]{\textwidth}
\centering
\caption{$\beta=1.0$}
\resizebox{0.92\textwidth}{!}{%
\begin{tabular}{llccccc}
\toprule
Method & Statistic & $N=20$ & $N=40$ & $N=60$ & $N=80$ & $N=100$ \\
\midrule
MNL-MLE* & exact sol.
& {100.00} \std{0.00}
& {100.00} \std{0.00}
& {100.00} \std{0.00}
& {100.00} \std{0.00}
& {99.99} \std{0.01} \\
\midrule
\multirow{4}{*}{Unguided}
& max
& 89.95 \std{6.12}
& 81.14 \std{7.54}
& 76.67 \std{4.82}
& 72.24 \std{3.35}
& 68.95 \std{2.72} \\
& mean
& 64.73 \std{6.34}
& 61.43 \std{8.16}
& 60.12 \std{4.26}
& 57.60 \std{5.78}
& 55.95 \std{2.83} \\
& \texttt{q50}
& 65.11 \std{6.71}
& 62.26 \std{8.59}
& 60.77 \std{4.83}
& 57.75 \std{6.08}
& 55.87 \std{2.91} \\
& \texttt{q90}
& 76.71 \std{5.43}
& 72.25 \std{8.22}
& 68.01 \std{4.87}
& 64.02 \std{5.84}
& 61.76 \std{2.81} \\
\midrule
\multirow{4}{*}{GAsN-guided}
& max
& \textbf{99.93} \std{0.12}
& \textbf{99.96} \std{0.13}
& \textbf{99.78} \std{0.16}
& 97.62 \std{6.07}
& 99.56 \std{0.28} \\
& mean
& 99.17 \std{1.79}
& \textbf{99.65} \std{0.41}
& \textbf{99.17} \std{0.44}
& 96.37 \std{7.75}
& \textbf{98.82} \std{0.61} \\
& \texttt{q50}
& 99.11 \std{1.98}
& \textbf{99.63} \std{0.46}
& 99.23 \std{0.44}
& 96.42 \std{7.52}
& 98.85 \std{0.66} \\
& \texttt{q90}
& 99.33 \std{1.56}
& \textbf{99.85} \std{0.27}
& 99.50 \std{0.34}
& 97.14 \std{6.93}
& 99.24 \std{0.45} \\
\midrule
\multirow{4}{*}{RAsN-guided}
& max
& 99.80 \std{0.55}
& 99.83 \std{0.35}
& 99.59 \std{1.09}
& \textbf{99.91} \std{0.14}
& \textbf{99.80} \std{0.28} \\
& mean
& \textbf{99.56} \std{0.55}
& 99.01 \std{0.97}
& 98.96 \std{2.25}
& \textbf{99.36} \std{0.41}
& 98.78 \std{0.65} \\
& \texttt{q50}
& \textbf{99.62} \std{0.58}
& 99.06 \std{0.96}
& \textbf{99.39} \std{1.04}
& \textbf{99.39} \std{0.40}
& \textbf{98.87} \std{0.77} \\
& \texttt{q90}
& \textbf{99.72} \std{0.55}
& 99.53 \std{0.73}
& \textbf{99.54} \std{1.07}
& \textbf{99.73} \std{0.25}
& \textbf{99.43} \std{0.50} \\
\bottomrule
\end{tabular}%
}
\end{subtable}

\vspace{0.8em}

\begin{subtable}[t]{\textwidth}
\centering
\caption{$\beta=0.1$}
\resizebox{0.92\textwidth}{!}{%
\begin{tabular}{llccccc}
\toprule
Method & Statistic & $N=20$ & $N=40$ & $N=60$ & $N=80$ & $N=100$ \\
\midrule
MNL-MLE* & exact sol.
& {100.00} \std{0.00}
& {99.99} \std{0.02}
& {100.00} \std{0.00}
& {100.00} \std{0.00}
& {99.99} \std{0.01} \\
\midrule
\multirow{4}{*}{Unguided}
& max
& 90.45 \std{5.54}
& 80.89 \std{7.65}
& 76.40 \std{4.96}
& 71.35 \std{5.95}
& 68.80 \std{3.99} \\
& mean
& 63.72 \std{6.73}
& 60.43 \std{7.78}
& 59.79 \std{4.27}
& 57.17 \std{5.79}
& 55.47 \std{2.93} \\
& \texttt{q50}
& 64.28 \std{7.18}
& 60.26 \std{7.95}
& 60.30 \std{4.54}
& 57.33 \std{5.91}
& 55.42 \std{2.95} \\
& \texttt{q90}
& 75.57 \std{6.07}
& 71.88 \std{7.71}
& 68.31 \std{4.92}
& 63.84 \std{5.55}
& 61.39 \std{2.82} \\
\midrule
\multirow{4}{*}{GAsN-guided}
& max
& \textbf{99.93} \std{0.12}
& \textbf{99.93} \std{0.14}
& 99.74 \std{0.04}
& 99.86 \std{0.16}
& 99.89 \std{0.18} \\
& mean
& \textbf{99.72} \std{0.39}
& \textbf{99.61} \std{0.51}
& {99.26} \std{0.42}
& 98.93 \std{0.61}
& {98.77} \std{0.67} \\
& \texttt{q50}
& \textbf{99.77} \std{0.33}
& \textbf{99.63} \std{0.53}
& {99.27} \std{0.46}
& 98.93 \std{0.65}
& {98.80} \std{0.73} \\
& \texttt{q90}
& \textbf{99.93} \std{0.12}
& \textbf{99.83} \std{0.34}
& {99.57} \std{0.32}
& 99.60 \std{0.55}
& 99.62 \std{0.33} \\
\midrule
\multirow{4}{*}{RAsN-guided}
& max
& 99.65 \std{0.81}
& 99.54 \std{0.68}
& \textbf{99.85} \std{0.28}
& \textbf{99.98} \std{0.03}
& \textbf{99.98} \std{0.03} \\
& mean
& 98.97 \std{1.05}
& 99.18 \std{0.75}
& \textbf{99.39} \std{0.37}
& \textbf{99.60} \std{0.16}
& \textbf{99.54} \std{0.25} \\
& \texttt{q50}
& 99.05 \std{1.07}
& 99.14 \std{0.78}
& \textbf{99.42} \std{0.47}
& \textbf{99.67} \std{0.23}
& \textbf{99.56} \std{0.31} \\
& \texttt{q90}
& 99.40 \std{0.69}
& 99.40 \std{0.69}
& \textbf{99.71} \std{0.43}
& \textbf{99.93} \std{0.07}
& \textbf{99.84} \std{0.15} \\
\bottomrule
\end{tabular}%
}
\end{subtable}
\end{table}

\begin{table}[htbp]
\centering
\caption{Optimal ratio statistics of generated assortments on MCCM data. All entries are reported in percentage units (\%) and shown as mean $\pm$ standard deviation over 10 runs. The ground-truth choice model is MCCM, and MCCM-EM is the correctly specified parametric baseline. For this baseline, we report the quality of the single solution returned by the method; for GAsN and RAsN-guided methods, we report summary statistics over generated 256 assortments. * marks the oracle parametric model (i.e., the model class matches the ground-truth choice model).}
\label{tab:mccm_opt_ratio}
\newcommand{\std}[1]{\ensuremath{\!{\scriptstyle \pm\, #1}}}
\begin{subtable}[h]{\textwidth}
\centering
\caption{$\beta=1.0$}
\resizebox{0.92\textwidth}{!}{%
\begin{tabular}{llccccc}
\toprule
Method & Statistic & $N=20$ & $N=40$ & $N=60$ & $N=80$ & $N=100$ \\
\midrule
MCCM-EM* & exact sol. 
& 99.34 \std{1.83} 
& 99.36 \std{0.71} 
& 99.50 \std{0.71} 
& 97.54 \std{6.74} 
& 99.89 \std{0.09} \\
\midrule
\multirow{4}{*}{Unguided}
& max  & 86.25 \std{4.90} & 78.53 \std{6.38} & 72.91 \std{4.10} & 69.58 \std{3.31} & 67.00 \std{2.15} \\
& mean & 62.97 \std{7.88} & 62.24 \std{6.41} & 60.32 \std{2.88} & 58.85 \std{2.95} & 56.88 \std{2.40} \\
& \texttt{q50}  & 63.22 \std{8.34} & 62.39 \std{6.43} & 60.24 \std{2.87} & 58.81 \std{2.89} & 56.94 \std{2.41} \\
& \texttt{q90}  & 74.57 \std{7.11} & 70.26 \std{6.20} & 66.48 \std{3.06} & 63.87 \std{2.73} & 61.41 \std{2.30} \\
\midrule
\multirow{4}{*}{GAsN-guided}
& max  
& {99.50} \std{1.49} 
& \textbf{99.92} \std{0.14} 
& \textbf{99.97} \std{0.04} 
& 99.82 \std{0.32} 
& {99.91} \std{0.11} \\
& mean 
& \textbf{99.42} \std{1.51} 
& \textbf{99.55} \std{0.46} 
& \textbf{99.67} \std{0.20} 
& {99.25} \std{0.48} 
& {99.29} \std{0.44} \\
& \texttt{q50}  
& \textbf{99.44} \std{1.48} 
& \textbf{99.54} \std{0.49} 
& \textbf{99.70} \std{0.22} 
& {99.27} \std{0.52} 
& {99.37} \std{0.43} \\
& \texttt{q90}  
& \textbf{99.48} \std{1.49} 
& \textbf{99.78} \std{0.32} 
& \textbf{99.88} \std{0.10} 
& {99.62} \std{0.41} 
& {99.63} \std{0.30} \\
\midrule
\multirow{4}{*}{RAsN-guided}
& max  
& \textbf{99.65} \std{0.46} 
& 99.60 \std{0.67} 
& 99.87 \std{0.22} 
& \textbf{99.99} \std{0.02} 
& \textbf{99.99} \std{0.02} \\
& mean 
& 98.95 \std{1.03} 
& 99.17 \std{0.67} 
& 99.45 \std{0.35} 
& \textbf{99.61} \std{0.18} 
& \textbf{99.59} \std{0.19} \\
& \texttt{q50}  
& 98.97 \std{1.28} 
& 99.23 \std{0.60} 
& 99.51 \std{0.36} 
& \textbf{99.68} \std{0.27} 
& \textbf{99.66} \std{0.18} \\
& \texttt{q90}  
& 99.35 \std{0.77} 
& 99.47 \std{0.66} 
& 99.71 \std{0.43} 
& \textbf{99.92} \std{0.10} 
& \textbf{99.88} \std{0.09} \\
\bottomrule
\end{tabular}%
}
\end{subtable}

\vspace{1em}

\begin{subtable}[h]{\textwidth}
\centering
\caption{$\beta=0.1$}
\resizebox{0.92\textwidth}{!}{%
\begin{tabular}{llccccc}
\toprule
Method & Statistic & $N=20$ & $N=40$ & $N=60$ & $N=80$ & $N=100$ \\
\midrule
MCCM-EM* & exact sol. 
& 98.84 \std{2.17} 
& 99.44 \std{0.69} 
& 99.50 \std{0.71} 
& 97.53 \std{6.74} 
& 99.84 \std{0.18} \\
\midrule
\multirow{4}{*}{Unguided}
& max  & 86.22 \std{5.42} & 78.12 \std{6.09} & 72.25 \std{3.87} & 69.71 \std{3.22} & 65.56 \std{2.40} \\
& mean & 63.20 \std{7.00} & 62.31 \std{6.51} & 59.83 \std{3.01} & 58.52 \std{2.95} & 56.47 \std{2.56} \\
& \texttt{q50}  & 63.56 \std{7.28} & 62.49 \std{6.61} & 59.77 \std{3.06} & 58.50 \std{2.97} & 56.49 \std{2.60} \\
& \texttt{q90}  & 74.59 \std{5.94} & 69.82 \std{6.31} & 65.90 \std{3.04} & 63.53 \std{2.90} & 60.73 \std{2.58} \\
\midrule
\multirow{4}{*}{GAsN-guided}
& max  
& {99.48} \std{1.49} 
& \textbf{99.93} \std{0.14} 
& \textbf{99.97} \std{0.04} 
& {99.86} \std{0.16} 
& {99.89} \std{0.18} \\
& mean 
& \textbf{99.41} \std{1.48} 
& \textbf{99.55} \std{0.51} 
& \textbf{99.64} \std{0.22} 
& {99.20} \std{0.48} 
& {99.25} \std{0.46} \\
& \texttt{q50}  
& \textbf{99.44} \std{1.48} 
& \textbf{99.60} \std{0.53} 
& \textbf{99.67} \std{0.26} 
& {99.23} \std{0.51} 
& {99.31} \std{0.47} \\
& \texttt{q90}  
& \textbf{99.48} \std{1.49} 
& \textbf{99.73} \std{0.34} 
& \textbf{99.87} \std{0.12} 
& {99.60} \std{0.37} 
& {99.62} \std{0.33} \\
\midrule
\multirow{4}{*}{RAsN-guided}
& max  
& \textbf{99.65} \std{0.46} 
& 99.54 \std{0.68} 
& 99.85 \std{0.28} 
& \textbf{99.98} \std{0.03} 
& \textbf{99.98} \std{0.03} \\
& mean 
& 98.97 \std{1.05} 
& 99.18 \std{0.75} 
& 99.39 \std{0.37} 
& \textbf{99.60} \std{0.16} 
& \textbf{99.54} \std{0.25} \\
& \texttt{q50}  
& 99.05 \std{1.07} 
& 99.14 \std{0.78} 
& 99.42 \std{0.47} 
& \textbf{99.67} \std{0.23} 
& \textbf{99.56} \std{0.31} \\
& \texttt{q90}  
& 99.40 \std{0.69} 
& 99.40 \std{0.69} 
& 99.71 \std{0.43} 
& \textbf{99.93} \std{0.07} 
& \textbf{99.84} \std{0.15} \\
\bottomrule
\end{tabular}%
}
\end{subtable}
\end{table}

\begin{table}[htbp]
\centering
\newcommand{\std}[1]{\ensuremath{\!{\scriptstyle \pm\, #1}}}
\caption{Optimal ratio statistics of generated assortments on MMNL data. All entries are reported in percentage units (\%) and shown as mean $\pm$ standard deviation over 10 runs. The ground-truth choice model is MMNL, while MNL-MLE and MCCM-EM are misspecified parametric baselines. For these baselines, we report the quality of the single solution returned by the method; for GAsN and RAsN-guided methods, we report summary statistics over generated 256 assortments.}
\label{tab:mmnl_opt_ratio}
\begin{subtable}[ht]{\textwidth}
\centering
\caption{$\beta=1.0$}
\resizebox{0.9\textwidth}{!}{%
\begin{tabular}{llccccc}
\toprule
Method & Statistic & $N=20$ & $N=40$ & $N=60$ & $N=80$ & $N=100$ \\
\midrule
MNL-MLE & exact sol.
& 83.64 \std{12.57}
& 54.10 \std{11.80}
& 63.22 \std{12.77}
& 81.92 \std{14.78}
& 91.09 \std{5.44} \\
\midrule
MCCM-EM & exact sol.
& 91.94 \std{9.64}
& 58.91 \std{16.22}
& 55.65 \std{11.30}
& 49.32 \std{11.86}
& 59.03 \std{17.56} \\
\midrule
\multirow{4}{*}{Unguided}
& max  & 86.52 \std{4.06} & 75.54 \std{4.15} & 73.52 \std{3.03} & 66.47 \std{3.56} & 65.37 \std{3.16} \\
& mean & 60.72 \std{6.15} & 55.66 \std{5.66} & 56.32 \std{3.31} & 53.21 \std{2.53} & 54.37 \std{3.20} \\
& \texttt{q50}  & 60.85 \std{6.62} & 55.50 \std{5.92} & 56.44 \std{3.41} & 53.31 \std{2.58} & 54.46 \std{3.26} \\
& \texttt{q90}  & 72.27 \std{6.47} & 64.92 \std{5.80} & 64.02 \std{3.57} & 59.24 \std{2.76} & 60.41 \std{3.44} \\
\midrule
\multirow{4}{*}{GAsN-guided}
& max
& \textbf{95.78} \std{5.45}
& \textbf{98.58} \std{1.43}
& \textbf{98.77} \std{0.61}
& 97.23 \std{1.45}
& 97.39 \std{0.66} \\
& mean
& \textbf{92.27} \std{8.10}
& \textbf{97.26} \std{1.76}
& \textbf{97.85} \std{0.81}
& 92.96 \std{8.51}
& 93.31 \std{5.99} \\
& \texttt{q50}
& \textbf{92.08} \std{8.26}
& \textbf{97.27} \std{1.70}
& \textbf{97.76} \std{0.94}
& 93.09 \std{8.59}
& 93.64 \std{6.04} \\
& \texttt{q90}
& \textbf{92.30} \std{8.44}
& \textbf{97.89} \std{1.61}
& \textbf{98.31} \std{0.70}
& 93.64 \std{8.58}
& 96.43 \std{1.26} \\
\midrule
\multirow{4}{*}{RAsN-guided}
& max
& 92.54 \std{9.98}
& 94.71 \std{6.17}
& 98.72 \std{1.14}
& \textbf{98.33} \std{1.04}
& \textbf{98.31} \std{0.79} \\
& mean
& 89.99 \std{9.99}
& 87.48 \std{9.70}
& 96.74 \std{2.58}
& \textbf{97.26} \std{1.53}
& \textbf{96.36} \std{2.08} \\
& \texttt{q50}
& 90.23 \std{9.49}
& 86.98 \std{9.60}
& 97.10 \std{1.62}
& \textbf{97.44} \std{1.65}
& \textbf{96.33} \std{2.17} \\
& \texttt{q90}
& 91.18 \std{10.02}
& 89.43 \std{10.82}
& 98.02 \std{1.75}
& \textbf{97.94} \std{1.44}
& \textbf{97.50} \std{1.48} \\
\bottomrule
\end{tabular}%
}
\end{subtable}

\vspace{1em}

\begin{subtable}[ht]{\textwidth}
\centering
\caption{$\beta=0.1$}
\resizebox{0.9\textwidth}{!}{%
\begin{tabular}{llccccc}
\toprule
Method & Statistic & $N=20$ & $N=40$ & $N=60$ & $N=80$ & $N=100$ \\
\midrule
MNL-MLE & exact sol.
& 84.10 \std{12.77}
& 56.15 \std{13.67}
& 65.72 \std{14.42}
& 79.80 \std{17.58}
& 89.62 \std{8.48} \\
\midrule
MCCM-EM & exact sol.
& 91.66 \std{8.46}
& 56.24 \std{16.40}
& 55.65 \std{11.30}
& 49.32 \std{11.86}
& 59.03 \std{17.56} \\
\midrule
\multirow{4}{*}{Unguided}
& max  & 89.25 \std{4.62} & 75.62 \std{4.33} & 72.15 \std{3.25} & 66.68 \std{3.00} & 65.77 \std{3.11} \\
& mean & 61.75 \std{5.56} & 56.13 \std{5.58} & 56.78 \std{2.98} & 53.52 \std{2.63} & 53.17 \std{3.16} \\
& \texttt{q50}  & 61.92 \std{5.89} & 56.02 \std{5.79} & 56.89 \std{3.08} & 53.58 \std{2.67} & 53.26 \std{3.19} \\
& \texttt{q90}  & 73.79 \std{6.04} & 65.63 \std{5.94} & 63.73 \std{3.36} & 59.67 \std{2.93} & 59.02 \std{3.23} \\
\midrule
\multirow{4}{*}{GAsN-guided}
& max
& \textbf{95.78} \std{5.45}
& \textbf{98.81} \std{1.34}
& 98.82 \std{0.56}
& 97.52 \std{1.35}
& 97.52 \std{0.69} \\
& mean
& \textbf{93.91} \std{5.67}
& \textbf{97.42} \std{1.52}
& 95.38 \std{7.80}
& 93.10 \std{7.97}
& 95.75 \std{1.59} \\
& \texttt{q50}
& \textbf{93.48} \std{5.40}
& \textbf{97.32} \std{1.78}
& 95.25 \std{8.25}
& 93.17 \std{8.33}
& 95.73 \std{1.65} \\
& \texttt{q90}
& \textbf{95.78} \std{5.45}
& \textbf{98.03} \std{1.45}
& 96.13 \std{6.85}
& 93.76 \std{8.20}
& 96.43 \std{1.32} \\
\midrule
\multirow{4}{*}{RAsN-guided}
& max
& 93.32 \std{10.20}
& 94.18 \std{7.44}
& \textbf{98.77} \std{1.15}
& \textbf{98.36} \std{1.02}
& \textbf{98.25} \std{0.83} \\
& mean
& 90.52 \std{9.35}
& 88.24 \std{9.64}
& \textbf{97.35} \std{1.14}
& \textbf{97.20} \std{1.52}
& \textbf{96.37} \std{1.95} \\
& \texttt{q50}
& 90.34 \std{9.58}
& 87.86 \std{9.64}
& \textbf{97.19} \std{1.48}
& \textbf{97.28} \std{1.64}
& \textbf{96.42} \std{2.05} \\
& \texttt{q90}
& 92.54 \std{9.98}
& 91.05 \std{11.03}
& \textbf{98.09} \std{1.46}
& \textbf{97.93} \std{1.46}
& \textbf{97.43} \std{1.38} \\
\bottomrule
\end{tabular}%
}
\end{subtable}
\end{table}

\begin{landscape}
\begin{table}[htbp]
\centering
\newcommand{\std}[1]{\ensuremath{\!{\scriptstyle \pm\, #1}}}
\caption{Optimal ratio statistics on MNL data across varying offline sample sizes. All entries are reported in percentage units (\%) and shown as mean $\pm$ standard deviation over 10 runs. For MNL-MLE, we report the quality of the single returned solution; for GAsN-guided and RAsN-guided methods, we report summary statistics over 256 generated assortments.}
\label{tab:sample_efficiency_mnl}
\resizebox{1.38\textwidth}{!}{%
\begin{tabular}{l l l ccccccccc}
\toprule
$N$ & Method & Statistic & $n=100$ & $n=200$ & $n=500$ & $n=1000$ & $n=2000$ & $n=5000$ & $n=10000$ & $n=20000$ & $n=50000$ \\
\midrule
\multirow{9}{*}{$40$} & MNL-MLE & exact sol. & 94.60 \std{8.87} & 99.55 \std{0.79} & 99.96 \std{0.08} & 99.96 \std{0.08} & 99.93 \std{0.10} & 100.00 \std{0.01} & 100.00 \std{0.00} & 100.00 \std{0.00} & 100.00 \std{0.01} \\
\cmidrule(lr){2-12}
 & \multirow{4}{*}{GAsN} & max & 99.98 \std{0.05} & 99.99 \std{0.04} & 100.00 \std{0.01} & 98.18 \std{5.40} & 99.98 \std{0.05} & 99.92 \std{0.15} & 99.96 \std{0.10} & 99.97 \std{0.09} & 99.85 \std{0.37} \\
 &  & mean & 97.31 \std{7.16} & 99.72 \std{0.26} & 99.75 \std{0.19} & 97.49 \std{6.81} & 97.75 \std{6.04} & 98.54 \std{3.39} & 99.66 \std{0.30} & 99.70 \std{0.39} & 99.01 \std{1.82} \\
 &  & \texttt{q50} & 97.24 \std{7.56} & 99.75 \std{0.35} & 99.80 \std{0.21} & 97.50 \std{6.88} & 97.52 \std{6.91} & 98.29 \std{4.27} & 99.69 \std{0.29} & 99.71 \std{0.41} & 98.93 \std{2.18} \\
 &  & \texttt{q90} & 97.60 \std{6.94} & 99.93 \std{0.11} & 99.93 \std{0.12} & 97.84 \std{6.31} & 99.88 \std{0.15} & 99.81 \std{0.29} & 99.90 \std{0.17} & 99.87 \std{0.26} & 99.55 \std{0.96} \\
\cmidrule(lr){2-12}
 & \multirow{4}{*}{RAsN} & max & 99.29 \std{0.85} & 99.29 \std{0.83} & 99.72 \std{0.62} & 99.65 \std{0.60} & 99.72 \std{0.62} & 99.72 \std{0.62} & 99.74 \std{0.62} & 99.67 \std{0.56} & 99.87 \std{0.32} \\
 &  & mean & 96.55 \std{2.91} & 96.63 \std{2.82} & 97.30 \std{2.01} & 97.90 \std{2.16} & 98.18 \std{1.75} & 99.11 \std{0.94} & 99.36 \std{0.75} & 99.12 \std{0.96} & 99.67 \std{0.53} \\
 &  & \texttt{q50} & 96.91 \std{2.51} & 96.73 \std{2.75} & 97.00 \std{2.35} & 98.41 \std{1.57} & 98.28 \std{1.97} & 99.10 \std{1.00} & 99.46 \std{0.71} & 99.31 \std{0.81} & 99.72 \std{0.61} \\
 &  & \texttt{q90} & 98.13 \std{2.07} & 97.69 \std{2.28} & 98.55 \std{1.62} & 98.64 \std{1.66} & 98.66 \std{1.67} & 99.45 \std{0.93} & 99.53 \std{0.74} & 99.56 \std{0.73} & 99.84 \std{0.35} \\
\midrule
\multirow{9}{*}{$80$} & MNL-MLE & exact sol. & 94.91 \std{9.07} & 99.52 \std{0.68} & 99.90 \std{0.14} & 99.96 \std{0.11} & 99.97 \std{0.08} & 100.00 \std{0.00} & 100.00 \std{0.00} & 100.00 \std{0.00} & 100.00 \std{0.00} \\
\cmidrule(lr){2-12}
 & \multirow{4}{*}{GAsN} & max & 99.97 \std{0.04} & 99.94 \std{0.11} & 97.61 \std{6.98} & 99.90 \std{0.14} & 99.87 \std{0.17} & 99.69 \std{0.30} & 99.64 \std{0.35} & 99.67 \std{0.32} & 99.84 \std{0.24} \\
 &  & mean & 99.49 \std{0.30} & 96.14 \std{10.10} & 96.71 \std{8.43} & 98.33 \std{3.56} & 99.41 \std{0.44} & 99.06 \std{0.64} & 98.94 \std{0.68} & 98.75 \std{0.83} & 98.89 \std{0.92} \\
 &  & \texttt{q50} & 99.56 \std{0.33} & 96.10 \std{10.42} & 96.78 \std{8.46} & 99.45 \std{0.48} & 99.48 \std{0.46} & 99.08 \std{0.72} & 99.00 \std{0.77} & 98.81 \std{0.79} & 99.05 \std{0.71} \\
 &  & \texttt{q90} & 99.82 \std{0.22} & 96.54 \std{9.88} & 97.22 \std{7.86} & 99.77 \std{0.29} & 99.71 \std{0.29} & 99.36 \std{0.51} & 99.34 \std{0.55} & 99.32 \std{0.48} & 99.59 \std{0.37} \\
\cmidrule(lr){2-12}
 & \multirow{4}{*}{RAsN} & max & 99.89 \std{0.17} & 99.86 \std{0.17} & 99.86 \std{0.17} & 99.81 \std{0.21} & 99.87 \std{0.17} & 99.92 \std{0.13} & 99.88 \std{0.20} & 99.90 \std{0.16} & 99.97 \std{0.07} \\
 &  & mean & 99.20 \std{0.76} & 99.18 \std{0.78} & 96.98 \std{6.60} & 97.22 \std{6.19} & 99.39 \std{0.64} & 99.53 \std{0.38} & 99.43 \std{0.34} & 99.32 \std{0.55} & 99.61 \std{0.27} \\
 &  & \texttt{q50} & 99.35 \std{0.82} & 99.34 \std{0.82} & 97.02 \std{6.80} & 97.22 \std{6.26} & 99.34 \std{0.82} & 99.58 \std{0.44} & 99.47 \std{0.33} & 99.40 \std{0.47} & 99.68 \std{0.29} \\
 &  & \texttt{q90} & 99.62 \std{0.53} & 99.63 \std{0.53} & 97.53 \std{6.19} & 97.70 \std{5.71} & 99.63 \std{0.53} & 99.80 \std{0.16} & 99.74 \std{0.27} & 99.70 \std{0.31} & 99.80 \std{0.29} \\
\bottomrule
\end{tabular}%
}
\end{table}
\end{landscape}

\begin{landscape}
\begin{table}[htbp]
\centering
\newcommand{\std}[1]{\ensuremath{\!{\scriptstyle \pm\, #1}}}
\caption{Optimal ratio statistics on MMNL data across varying offline sample sizes. All entries are reported in percentage units (\%) and shown as mean $\pm$ standard deviation over 10 runs. For MNL-MLE, we report the quality of the single returned solution; for GAsN-guided and RAsN-guided methods, we report summary statistics over 256 generated assortments.}
\label{tab:sample_efficiency_mmnl}
\resizebox{1.39\textwidth}{!}{%
\begin{tabular}{l l l ccccccccc}
\toprule
$N$ & Method & Statistic & $n=100$ & $n=200$ & $n=500$ & $n=1000$ & $n=2000$ & $n=5000$ & $n=10000$ & $n=20000$ & $n=50000$ \\
\midrule
\multirow{9}{*}{$40$} & MNL-MLE & exact sol. & 57.07 \std{12.12} & 54.30 \std{13.13} & 56.04 \std{15.97} & 58.20 \std{16.83} & 58.20 \std{16.83} & 58.20 \std{16.83} & 54.10 \std{11.80} & 56.56 \std{14.67} & 56.56 \std{14.67} \\
\cmidrule(lr){2-12}
 & \multirow{4}{*}{GAsN} & max & 99.12 \std{0.52} & 99.19 \std{0.53} & 99.26 \std{0.55} & 99.26 \std{0.55} & 99.13 \std{0.92} & 99.12 \std{0.91} & 99.13 \std{1.09} & 95.74 \std{4.41} & 83.28 \std{11.67} \\
 &  & mean & 97.84 \std{1.41} & 97.88 \std{1.42} & 97.82 \std{1.47} & 97.93 \std{1.22} & 95.24 \std{7.65} & 97.51 \std{2.25} & 97.77 \std{1.58} & 88.00 \std{11.42} & 73.81 \std{16.65} \\
 &  & \texttt{q50} & 97.76 \std{1.58} & 97.94 \std{1.70} & 97.74 \std{1.68} & 97.81 \std{1.27} & 95.61 \std{6.93} & 98.14 \std{1.73} & 98.29 \std{1.32} & 88.04 \std{11.90} & 73.62 \std{18.33} \\
 &  & \texttt{q90} & 98.52 \std{1.16} & 98.36 \std{1.32} & 98.43 \std{1.20} & 98.58 \std{1.25} & 96.16 \std{6.88} & 98.78 \std{1.29} & 98.69 \std{1.18} & 88.87 \std{11.85} & 77.63 \std{15.16} \\
\cmidrule(lr){2-12}
 & \multirow{4}{*}{RAsN} & max & 99.28 \std{1.03} & 99.02 \std{0.99} & 99.06 \std{1.02} & 99.06 \std{1.02} & 99.06 \std{1.02} & 96.41 \std{5.51} & 96.09 \std{7.10} & 90.61 \std{10.79} & 82.75 \std{18.14} \\
 &  & mean & 97.23 \std{1.28} & 96.50 \std{2.55} & 95.92 \std{3.16} & 96.65 \std{2.13} & 94.15 \std{6.21} & 93.63 \std{5.68} & 92.63 \std{7.31} & 87.37 \std{9.40} & 78.15 \std{16.66} \\
 &  & \texttt{q50} & 97.56 \std{1.58} & 96.77 \std{2.29} & 95.91 \std{3.67} & 96.88 \std{2.25} & 94.43 \std{6.09} & 94.49 \std{6.36} & 93.18 \std{6.93} & 87.48 \std{9.69} & 77.46 \std{16.91} \\
 &  & \texttt{q90} & 98.71 \std{1.14} & 98.69 \std{1.17} & 98.45 \std{1.57} & 98.69 \std{1.17} & 98.10 \std{1.57} & 96.31 \std{5.46} & 94.65 \std{7.35} & 88.93 \std{10.51} & 80.56 \std{17.02} \\
\midrule
\multirow{9}{*}{$80$} & MNL-MLE & exact sol. & 58.80 \std{12.01} & 56.75 \std{16.20} & 60.94 \std{16.17} & 77.09 \std{13.20} & 77.72 \std{14.88} & 77.76 \std{17.25} & 76.81 \std{17.88} & 81.17 \std{11.74} & 78.14 \std{13.68} \\
\cmidrule(lr){2-12}
 & \multirow{4}{*}{GAsN} & max & 98.35 \std{0.92} & 98.35 \std{0.98} & 98.24 \std{1.06} & 98.35 \std{0.91} & 97.96 \std{1.19} & 97.39 \std{1.35} & 97.66 \std{1.43} & 97.78 \std{1.42} & 98.80 \std{1.19} \\
 &  & mean & 96.53 \std{1.42} & 96.48 \std{1.45} & 96.57 \std{1.46} & 96.47 \std{1.49} & 96.50 \std{1.55} & 95.97 \std{1.49} & 96.07 \std{1.69} & 95.71 \std{2.05} & 97.21 \std{1.39} \\
 &  & \texttt{q50} & 96.45 \std{1.56} & 96.40 \std{1.63} & 96.51 \std{1.56} & 96.46 \std{1.59} & 96.50 \std{1.69} & 95.93 \std{1.52} & 96.18 \std{1.81} & 95.88 \std{1.77} & 97.15 \std{1.87} \\
 &  & \texttt{q90} & 97.34 \std{1.33} & 97.33 \std{1.28} & 97.37 \std{1.29} & 97.34 \std{1.46} & 97.05 \std{1.47} & 96.68 \std{1.53} & 96.75 \std{1.85} & 96.97 \std{1.75} & 98.37 \std{1.09} \\
\cmidrule(lr){2-12}
 & \multirow{4}{*}{RAsN} & max & 99.08 \std{1.20} & 98.76 \std{1.49} & 99.08 \std{1.20} & 98.79 \std{1.44} & 98.85 \std{1.11} & 98.35 \std{1.10} & 98.74 \std{0.78} & 98.92 \std{1.17} & 97.26 \std{2.78} \\
 &  & mean & 96.93 \std{1.75} & 96.86 \std{1.73} & 97.07 \std{1.57} & 96.94 \std{1.72} & 97.19 \std{1.47} & 97.17 \std{1.48} & 97.21 \std{1.24} & 97.51 \std{1.24} & 95.43 \std{2.91} \\
 &  & \texttt{q50} & 97.38 \std{1.30} & 97.29 \std{1.57} & 97.42 \std{1.29} & 97.22 \std{1.68} & 97.38 \std{1.37} & 97.03 \std{1.70} & 97.03 \std{1.44} & 97.62 \std{1.45} & 95.60 \std{2.95} \\
 &  & \texttt{q90} & 98.08 \std{1.50} & 98.08 \std{1.49} & 98.07 \std{1.47} & 98.03 \std{1.48} & 97.99 \std{1.46} & 97.97 \std{1.35} & 98.14 \std{1.11} & 98.20 \std{1.04} & 96.45 \std{3.37} \\
\bottomrule
\end{tabular}%
}
\end{table}
\end{landscape}

\section{Theoretical Proofs}
\subsection{Proof of Theorem \ref{thm:entreg}}\label{sec:entreg}
\begin{proof}[Proof of Theorem \ref{thm:entreg}]
We write the objective function as
\begin{equation*}
\Phi(q)
:=
\sum_{s \in \mathscr{S}} q(s) R(s)
+\frac{1}{\beta}\mathcal{H}(q)
=
\sum_{s \in \mathscr{S}} q(s) R(s)
-\frac{1}{\beta}\sum_{s \in \mathscr{S}} q(s)\log q(s),
\end{equation*}
where \(q \in \Delta(\mathscr{S})\), i.e.,
\begin{equation*}
q(s)\ge 0 \quad \text{for all } s \in \mathscr{S},
\qquad
\sum_{s \in \mathscr{S}} q(s)=1.
\end{equation*}
We wish to solve
\begin{equation*}
\max_{q \in \Delta(\mathscr{S})} \Phi(q).
\end{equation*}
We first note that the objective is \textit{strictly concave} in \(q\), because the term
\(\sum_s q(s)R(s)\) is linear in \(q\), while the entropy term
\(-\sum_s q(s)\log q(s)\) is strictly concave on the probability simplex.
Therefore, \(\Phi(q)\) is strictly concave, and hence it admits at most one
maximizer. It thus suffices to characterize the stationary point. To enforce the normalization constraint \(\sum_s q(s)=1\), consider the
Lagrangian
\begin{equation*}
\mathcal{L}(q,\lambda)
=
\sum_{s \in \mathscr{S}} q(s)R(s)
-\frac{1}{\beta}\sum_{s \in \mathscr{S}} q(s)\log q(s)
+\lambda\left(\sum_{s \in \mathscr{S}} q(s)-1\right).
\end{equation*}
For each \(s \in \mathscr{S}\), differentiate \(\mathcal{L}\) with respect to \(q(s)\). Using $\frac{d}{dx}(x\log x)=1+\log x,$
we obtain
\begin{equation*}
\frac{\partial \mathcal{L}}{\partial q(s)}
=
R(s)-\frac{1}{\beta}(1+\log q(s))+\lambda.
\end{equation*}
At an optimum, the first-order condition gives
\begin{equation*}
R(s)-\frac{1}{\beta}(1+\log q(s))+\lambda = 0.
\end{equation*}
Rearranging yields
\begin{equation*}
\log q(s)
=
\beta R(s)+\beta\lambda-1.
\end{equation*}
Hence there exists a constant \(C>0\), independent of \(s\), such that
\begin{equation*}
\log q(s)=\beta R(s)+C.
\end{equation*}
Exponentiating both sides gives
\begin{equation*}
q(s)=e^{C} e^{\beta R(s)}.
\end{equation*}
Now use the constraint that \(q\) is a probability distribution:
\begin{equation*}
1
=
\sum_{s \in \mathscr{S}} q(s)
=
e^{C}\sum_{s \in \mathscr{S}} e^{\beta R(s)}.
\end{equation*}
Therefore,
\begin{equation*}
e^{C}
=
\left(\sum_{s' \in \mathscr{S}} e^{\beta R(s')}\right)^{-1}.
\end{equation*}
Substituting this back into the expression for \(q(s)\), we obtain
\begin{equation*}
q^\star(s)
=
\frac{e^{\beta R(s)}}{\sum_{s' \in \mathscr{S}} e^{\beta R(s')}},\quad s\in\mathscr{S}.
\end{equation*}
Finally, since \(\Phi\) is strictly concave over \(\Delta(\mathscr{S})\), this
stationary point $q^\star$ is the unique maximizer. This completes the proof.
\end{proof}
\subsection{Proof of Proposition \ref{prop:forward_mixing}}\label{sec:pf_fwmix}
\begin{proof}[Proof of Proposition \ref{prop:forward_mixing}]
We first analyze one coordinate. Fix $i\in[N]$ and define
\[
m_{t,i}
:=
\mathbb{P}(s_{t,i}=1\mid s_0).
\]
By the definition of the corruption kernel, conditional on $s_{t-1,i}$, the next bit is kept with probability $1-\beta_t$ and is resampled uniformly from $\{0,1\}$ with probability $\beta_t$. Therefore,
\[
m_{t,i}
=
(1-\beta_t)m_{t-1,i}
+
\frac{\beta_t}{2}.
\]
Subtracting $1/2$ from both sides gives
\[
m_{t,i}-\frac{1}{2}
=
(1-\beta_t)
\left(
m_{t-1,i}-\frac{1}{2}
\right).
\]
Iterating this recursion yields
\[
m_{t,i}-\frac{1}{2}
=
\left(\prod_{\tau=1}^t(1-\beta_\tau)\right)
\left(s_{0,i}-\frac{1}{2}\right)
=
\bar\alpha_t
\left(s_{0,i}-\frac{1}{2}\right).
\]
Hence
\[
\mathbb{P}(s_{t,i}=1\mid s_0)
=
\frac{1}{2}
+
\bar\alpha_t
\left(s_{0,i}-\frac{1}{2}\right).
\]
Since the forward corruption acts independently across coordinates, conditional on $s_0$ the coordinates of $s_t$ are independent. Thus,
\[
q_t(s_t\mid s_0)
=
\prod_{i=1}^N
\mathbb{P}(s_{t,i}\mid s_0).
\]
Using the expression above, each coordinate marginal can be written as
\[
\mathbb{P}(s_{t,i}\mid s_0)
=
\frac{1}{2}
+
\bar\alpha_t
\left(s_{0,i}-\frac{1}{2}\right)(2s_{t,i}-1),
\]
which proves the factorized expression.

Finally, let $q_{t,i}$ denote the marginal law of $s_{t,i}$ given $s_0$, and let $u_i$ denote the uniform distribution on $\{0,1\}$. Then
\[
\|q_{t,i}-u_i\|_{\mathrm{TV}}
=
\left|
\mathbb{P}(s_{t,i}=1\mid s_0)-\frac{1}{2}
\right|
=
\bar\alpha_t
\left|s_{0,i}-\frac{1}{2}\right|
=
\frac{\bar\alpha_t}{2}.
\]
Using the tensorization bound
\[
\left\|
\bigotimes_{i=1}^N q_{t,i}
-
\bigotimes_{i=1}^N u_i
\right\|_{\mathrm{TV}}
\leq
\sum_{i=1}^N
\|q_{t,i}-u_i\|_{\mathrm{TV}},
\]
we obtain
\[
\left\|
q_t(\cdot\mid s_0)-u
\right\|_{\mathrm{TV}}
\leq
\sum_{i=1}^N
\frac{\bar\alpha_t}{2}
=
\frac{N}{2}\bar\alpha_t.
\]
This completes the proof.
\end{proof}

\subsection{Proof of Proposition \ref{prop:population_denoising}}\label{sec:pf_population_denoising}
\begin{proof}
The loss decomposes over coordinates. Fix a coordinate $i$ and condition on a particular value $(s_t,t)=(z,\tau)$.
Let
\[
\eta_i(z,\tau)
:=
\mathbb{P}(s_{0,i}=1\mid s_t=z,t=\tau).
\]
For a predicted probability $a\in(0,1)$, the conditional binary cross-entropy risk is
\[
r(a;z,\tau)
=
-\eta_i(z,\tau)\log a
-
(1-\eta_i(z,\tau))\log(1-a).
\]
This is minimized over $a\in(0,1)$ at $a=\eta_i(z,\tau)$. Indeed,
\[
\frac{\partial r}{\partial a}
=
-\frac{\eta_i(z,\tau)}{a}
+
\frac{1-\eta_i(z,\tau)}{1-a},
\]
and setting this derivative equal to zero gives
\[
a=\eta_i(z,\tau).
\]
Moreover, the second derivative is
\[
\frac{\partial^2 r}{\partial a^2}
=
\frac{\eta_i(z,\tau)}{a^2}
+
\frac{1-\eta_i(z,\tau)}{(1-a)^2}
>0,
\]
so this minimizer is unique whenever $\eta_i(z,\tau)\in(0,1)$.

Since $a=\sigma(g_i(z,\tau))$, the population minimizer satisfies
\[
\sigma(g_i^\star(z,\tau))
=
\eta_i(z,\tau)
=
\mathbb{P}(s_{0,i}=1\mid s_t=z,t=\tau).
\]
Taking the logit transform gives
\[
g_i^\star(z,\tau)
=
\log
\frac{\eta_i(z,\tau)}{1-\eta_i(z,\tau)}
=
\log
\frac{
\mathbb{P}(s_{0,i}=1\mid s_t=z,t=\tau)
}{
\mathbb{P}(s_{0,i}=0\mid s_t=z,t=\tau)
},
\]
whenever the posterior marginal lies strictly between $0$ and $1$.
\end{proof}

\subsection{Proof of Theorem \ref{thm:kl_guidance}}\label{sec:pf_klguide}
\begin{proof}[Proof of Theorem \ref{thm:kl_guidance}]
Fix $s_t$ and $t$. For notational simplicity, we write
\[
p_\phi(x) := p_\phi(x\mid s_t,t),
\qquad
\ell(x) := \ell_{s_t}(x).
\]
The local improvement objective is
\[
\mathcal{J}(q)
=
\mathbb{E}_{x\sim q}[\ell(x)]
-
\frac{1}{\lambda_t}
\mathrm{KL}(q\,\|\,p_\phi).
\]
Expanding the KL divergence gives
\[
\mathcal{J}(q)
=
\sum_{x}q(x)\ell(x)
-
\frac{1}{\lambda_t}
\sum_x q(x)\log\frac{q(x)}{p_\phi(x)}.
\]
Rearranging terms, we obtain
\[
\mathcal{J}(q)
=
-\frac{1}{\lambda_t}
\sum_x q(x)
\log
\frac{
q(x)
}{
p_\phi(x)\exp(\lambda_t \ell(x))
}.
\]
Define the normalizing constant
\[
Z
:=
\sum_{x\in\{0,1\}^N}
p_\phi(x)\exp(\lambda_t \ell(x)).
\]
Then
\[
\bar q(x)
:=
\frac{
p_\phi(x)\exp(\lambda_t \ell(x))
}{
Z
}
\]
is a probability distribution over $\{0,1\}^N$. Using this definition, we can rewrite the objective as
\[
\mathcal{J}(q)
=
\frac{1}{\lambda_t}\log Z
-
\frac{1}{\lambda_t}
\mathrm{KL}(q\,\|\,\bar q).
\]
Since $\mathrm{KL}(q\,\|\,\bar q)\geq 0$, with equality if and only if $q=\bar q$, the unique maximizer is
\[
q_t^\star(x\mid s_t)
=
\bar q(x)
=
\frac{
p_\phi(x\mid s_t,t)
\exp(\lambda_t \ell_{s_t}(x))
}{
Z_t(s_t)
}.
\]
This proves the first claim.

It remains to show that this optimizer factorizes and has the claimed coordinate-wise logits. By assumption,
\[
p_\phi(x\mid s_t,t)
=
\prod_{i=1}^N
\sigma(g_\phi^{(i)}(s_t,t))^{x_i}
\left(1-\sigma(g_\phi^{(i)}(s_t,t))\right)^{1-x_i}.
\]
Also, since
\[
\ell_{s_t}(x)
=
\sum_{i=1}^N x_i\Delta_i R(s_t),
\]
we have
\[
\exp(\lambda_t \ell_{s_t}(x))
=
\prod_{i=1}^N
\exp\left(\lambda_t x_i\Delta_i R(s_t)\right).
\]
Therefore,
\[
q_t^\star(x\mid s_t)
\propto
\prod_{i=1}^N
\left[
\sigma(g_\phi^{(i)}(s_t,t))^{x_i}
\left(1-\sigma(g_\phi^{(i)}(s_t,t))\right)^{1-x_i}
\exp\left(\lambda_t x_i\Delta_i R(s_t)\right)
\right].
\]
Hence $q_t^\star(\cdot\mid s_t)$ factorizes across coordinates. For each coordinate $i$, the odds ratio under $q_t^\star$ is
\[
\frac{
q_t^\star(x_i=1\mid s_t)
}{
q_t^\star(x_i=0\mid s_t)
}
=
\frac{
\sigma(g_\phi^{(i)}(s_t,t))
}{
1-\sigma(g_\phi^{(i)}(s_t,t))
}
\exp\left(\lambda_t\Delta_i R(s_t)\right).
\]
Taking logarithms gives
\[
\log
\frac{
q_t^\star(x_i=1\mid s_t)
}{
q_t^\star(x_i=0\mid s_t)
}
=
\log
\frac{
\sigma(g_\phi^{(i)}(s_t,t))
}{
1-\sigma(g_\phi^{(i)}(s_t,t))
}
+
\lambda_t\Delta_i R(s_t).
\]
Since $\log\frac{\sigma(a)}{1-\sigma(a)}=a$, we obtain
\[
\log
\frac{
q_t^\star(x_i=1\mid s_t)
}{
q_t^\star(x_i=0\mid s_t)
}
=
g_\phi^{(i)}(s_t,t)
+
\lambda_t\Delta_i R(s_t).
\]
Equivalently,
\[
q_t^\star(x_i=1\mid s_t)
=
\sigma\left(
g_\phi^{(i)}(s_t,t)
+
\lambda_t\Delta_i R(s_t)
\right).
\]
This is exactly the reward-guided reverse transition.
\end{proof}

\end{document}